\documentclass{article}

\usepackage{arxiv}

\usepackage[utf8]{inputenc} 
\usepackage[T1]{fontenc}    
\usepackage{hyperref}       
\usepackage{url}            
\usepackage{booktabs}       
\usepackage{amsfonts}       
\usepackage{nicefrac}       
\usepackage{microtype}      
\usepackage{lipsum}
\usepackage{graphicx}
\usepackage{subfigure}
\usepackage{algorithm}
\usepackage{algorithmic}
\usepackage{amsthm}
\graphicspath{ {./images/} }

\title{Meta-Regularization: An Approach to Adaptive Choice of the Learning Rate in Gradient Descent}

\usepackage{amsmath,amsfonts,bm}
\usepackage{amssymb}

\def\eqref#1{equation~\ref{#1}}

\def\1{\bm{1}}

\def\vone{{\bm{1}}}
\def\vmu{{\bm{\mu}}}
\def\veta{{\bm{\eta}}}
\def\valpha{{\bm{\alpha}}}
\def\vbeta{{\bm{\beta}}}

\def\va{{\bm{a}}}
\def\vb{{\bm{b}}}

\def\vg{{\bm{g}}}

\def\vu{{\bm{u}}}
\def\vv{{\bm{v}}}

\def\vx{{\bm{x}}}
\def\vy{{\bm{y}}}

\def\gA{{\mathcal{A}}}

\def\gO{{\mathcal{O}}}

\def\gX{{\mathcal{X}}}

\def\sR{{\mathbb{R}}}

\DeclareMathOperator*{\argmax}{arg\,max}
\DeclareMathOperator*{\argmin}{arg\,min}

\newcommand{\norm}[1]{\left\| #1 \right\|_2^2}
\newcommand{\diag}{\mathrm{diag}}

\newtheorem{thm}{Theorem}
\newtheorem{defn}[thm]{Definition}

\newtheorem{lemma}[thm]{Lemma}

\newtheorem{remark}{Remark}

\author{
 Guangzeng Xie, Hao Jin \& Dachao Lin \\
  Peking University\\
  Beijing, China\\
  \texttt{\{smsxgz, jin.hao, lindachao\}@pku.edu.cn} \\
   \And
 Zhihua Zhang \\
  Peking University\\
  Beijing, China\\
  \texttt{zhzhang@pku.edu.cn} \\
}

\begin{document}
\maketitle
\begin{abstract}
We propose \textit{Meta-Regularization}, a novel approach for the adaptive choice of the learning rate in first-order gradient descent methods. 
Our approach modifies the objective function by adding a regularization term on the learning rate, and casts the joint updating process of parameters and learning rates into a maxmin problem.
Given any regularization term, our approach facilitates the generation of practical algorithms.
When \textit{Meta-Regularization} takes the $\varphi$-divergence as a regularizer, the resulting algorithms exhibit comparable theoretical convergence performance with other first-order gradient-based algorithms.
Furthermore, we theoretically prove that some well-designed regularizers can improve the convergence performance under the strong-convexity condition of the objective function.
Numerical experiments on benchmark problems demonstrate the effectiveness of algorithms derived from some common $\varphi$-divergence in full batch as well as online learning settings.
\end{abstract}

\section{Introduction}
The automatic choice of the learning rate remains crucial in improving the efficiency of gradient descent algorithms.
Strategies regardless of training information, such as the learning rate decay, might drive the learning rate too large or too small during the training process, which tends to negatively affect the convergence performance.
In order to improve performance, adaptively updating the learning rate during the training process would be  desirable. 

There are two common approaches for updating the learning rate in the first-order gradient descent methods in the literature. 
Firstly, line search for a proper learning rate is a natural and direct approach to fully utilizing the currently received gradient information.
There are practically many ways to carry out such exact or inexact line search.
Specifically, Hyper-Gradient Descent \cite{gunes2018online} can be viewed as an approximate line search through using the gradient with respect to the learning rate of the update rule itself.
Secondly, some methods leverage historical gradients to approximate the inverse of Hessian matrix, which is essential in Newton method \cite{nocedal2006numerical}.
Quasi-Newton Methods \cite{liu1989limited} as well as  the Barzilai-Broweinin (BB) \cite{barzilai1988two} method all fall in the scope of such algorithms.

In this paper we propose a novel approach to the adaptive choice for the learning rate that we call  \textit{Meta-Regularization}. The key idea is to impose some constraints on the updates of learning rate during the training process, which is equivalent to adding a regularization term on the learning rate to the objective function. 
Through introducing a regularization term on the learning rate, our approach casts the joint updating process of parameters and learning rates into a maxmin problem.
In other words, our approach gives a pipeline to generate practical algorithms from any regularization term. 
Various regularization terms bring out various strategies of updating learning rate, which include AdaGrad \cite{duchi2011adagrad} and WNGrad \cite{wu2018wngrad}.
Compared with the Hyper-Gradient and BB methods, our approach 
is  attractive due to its ability in construction and theoretical analysis of the corresponding algorithms.  

Taking the regularization term derived from the $\varphi$-divergence as an instance, we theoretically analyze the overall performance of the resulting algorithms, and evaluate some representative algorithms on benchmark problems.
Theoretical guarantees of these algorithms are provided in both full batch and online learning settings, which are not explicitly given in the original work of Hyper-Gradient Descent.
Moreover, certain modifications of regularization terms from the $\varphi$-divergence manage to improve the theoretical convergence performance while the original objective function is strongly convex.
In terms of numerical experiments, we generate several algorithms from some common $\varphi$-divergence without delicate design to represent the general performance of such algorithms.
Experimental results not only reveal a generally comparable performance with Hyper-Gradient Descent as well as BB method, but also demonstrate outperformance over these two algorithms in online learning and full batch settings, respectively.

The main contributions of our paper are as follows:
\begin{itemize}
\vspace{-0.1in}
\setlength{\itemsep}{0pt}
\setlength{\parsep}{0pt}
    \item To our  knowledge, we are the first to formally consider the usage of regularization technique in adaptively updating the learning rate,  giving rise to a pipeline to construct algorithms from any given regularization term.
    \item We provide theoretical analysis of the convergence performance for a family of algorithms derived from our approach when taking a generalized distance function such as the $\varphi$-divergence  as the regularizer.
    \item Experimental results demonstrate that our Meta-Regularization method based on the $\varphi$-divergence is practically comparable with the BB method and Hyper-Gradient Descent, and even outperforms them in some cases. 
\end{itemize}
\section{Related Work}
Steepest Descent uses the received gradient direction and an exact or inexact line search to obtain proper learning rates. 
Although Steepest Descent uses the direction that descends most and the best learning rate that gives the most reduction of objective function value, 
Steepest Descent may converge very slow for convex quadratic functions when the Hessian matrix is ill-conditioned  \cite{yuan2008step}.
In practice, some line search conditions such as Goldstein conditions or Wolfe conditions \cite{fletcher2013practical} can be applied to compute the learning rate.
In online or stochastic settings, one observes stochastic gradients rather than exact gradients and line search methods become less effective. 

The BB method \cite{barzilai1988two} which was motivated by quasi-Newton methods presents a surprising result that it could lead to superlinear convergence in convex quadratic problem of two variables. 
Although numerical results often show that the BB method converges superlinearly in solving nonlinear optimization problems, no superlinear convergence results have been established even for an $n$-dimensional strictly convex quadratic problem with the order $n > 2$ \cite{barzilai1988two, dai2013new}.
In minimizing the sum of cost functions and stochastic setting, SGD-BB proposed by \cite{tan2016barzilai} takes the average of the stochastic gradients in one epoch as an estimation of the full gradient. 
But this approach can not directly be applied to online learning settings.

In online convex optimization \cite{zinkevich2003online, shalev2012online, hazan2016introduction}, AdaGrad  adapts the learning rate on per parameter basis dynamically. 
This leads to many variants such as RMSProp \cite{tieleman2012lecture}, 
AdaDelta \cite{zeiler2012adadelta}, Adam \cite{kingma2015adam},  etc.

Additionally, \cite{cruz2011almost} analyzed Adaptive Stochastic Gradient Descent (ASGD)
which is a generalization of Kesten's accelerated stochastic approximation algorithm 
\cite{kesten1958accelerated} for the high-dimensional case. 
ASGD uses a monotone decreasing function with respect to a time variable to get learning rates.
Recently, \cite{gunes2018online} proposed Hyper-Gradient Descent to learn the global learning rate in SGD, 
SGD with Nesterov momentum and Adam. 
Hyper-Gradient Descent can be viewed as an approximate line search method in the online learning setting
and it uses the update rule for the previous step to optimize the leaning rate in the current step. 
However, Hyper-Gradient Descent has no theoretical guarantee.

It is worth mentioning that \cite{gupta2017unified} proposed a framework 
named Unified Adaptive Regularization 
from which AdaGrad and Online Newton Step \cite{hazan2007logarithmic} can be derived.
However, Unified Adaptive Regularization gives an  approach for approximating the Hessian matrix in second order methods.     

Our framework stems from the work of \cite{daubechies2010iteratively}, who adjusted the weights of the weighted least squares problem by solving an extra objective function which adds a regularizer about the weights to origin objective function.
\section{Problem Formulation}
Before introducing our approach, we present the notation that will be used.
We denote the set $\{x > 0: x \in \sR\}$ by $\sR_{++}$. 
For two vectors $\va, \vb \in \sR^d$, we use $\va / \vb$ to denote element-wise division, 
$\va \circ \vb$ for element-wise product (the symbol $\circ$ will be omitted in the explicit context), 
$\va^n = (a_1^n, a_2^n, \ldots, a_d^n) $, and
$\va \ge \vb$ if $a_j \ge b_j$ for all $j$. 
Let $\vone$ be the vector of ones with an appropriate size,
and 
$\diag(\vbeta)$ be a diagonal matrix with the elements of the vector $\vbeta$ on the main diagonal.
In addition, we define $\|\va\|_A = \sqrt{\langle\va, A \va \rangle}$ where $A$ is a positive semidefinite matrix. 

Given a set ${\gX} \subseteq \sR^d$, a function $f\colon \mathcal{X} \to \sR$ is said to satisfy $f \in C_L^{1, 1}(\gX)$ if 
$f$ is continuously differentiable on $\gX$, and the derivative of $f$ is Lipschitz continuous on $\gX$ with constant $L$:
\[
    \| \nabla f(\vx) - \nabla f(\vy) \|_2 \le L \|\vx - \vy \|_2.
\]
More general definition can be found in \cite{nesterov2013introductory}.

We now give the notion of the $\varphi$-divergence.
\begin{defn}[$\varphi$-divergence]
    \label{def-phi-divergence}
    Let $\varphi$: $\sR_{++} \rightarrow \sR$ be a differentiable strongly convex function in $\sR_{++}$
    such that $\varphi (1) = \varphi' (1) = 0$, where $\varphi'$ is the derivative function of $\varphi$.
    Given such a function $\varphi$,
    the function $D_{\varphi}$: $\sR_{++}^{d} \times \sR_{++}^{d} \rightarrow \sR$,
    which is define by
    $$D_{\varphi}(\vu, \vv) \triangleq \sum_{j=1}^d \frac{1}{v_j} \varphi\left(\frac{v_j}{u_j}\right),$$
    is referred to as the $\varphi$-divergence.
\end{defn}
\begin{remark}
    Note that convex function $\varphi$ with $\varphi(1) = \varphi'(1) = 0$ satisfies $\varphi(z) \ge 0$ ~ for all $z > 0$,
    thus $D_{\varphi} (\vu, \vv)\ge 0$ ~for all $\vu, \vv \in \sR_{++}^d$, with equality iff $\vu=\vv$.
\end{remark}
\begin{remark}
    For any convex function $f$, $\varphi(z) = f(z) - f'(1)(z - 1) - f(1)$ is a proper function for our $\varphi$-divergence.
\end{remark}

For an online learning problem, a learner faces a sequence of 
convex functions $\{f_t\}$ with the same domain $\mathcal{X} \subseteq \sR^d$,
receives (sub)gradient information $\vg_t \in \partial f_t(\vx_t)$ at each step $t$,
and predicts a point $\vx_{t+1} \in \mathcal{X}$. 

In this setting, our main focus is the regret \cite{duchi2011adagrad, kingma2015adam}:
\begin{equation}
    \label{regret}
     R(T) = \sum_{t=0}^{T-1}f_t(\vx_t) - \min_{\vx \in \mathcal{X}} \sum_{t=0}^{T-1}f_t(\vx).
\end{equation}

In theoretical analysis, another important setting we consider is the full batch setting. Under this setting, we deal with a certain objective function $F$ with exact gradient at each step,
i.e., $f_t = F$. Moreover, the objective function $F$ satisfies $F \in C_L^{1,1}$ and does not have to be convex. Furthermore, we describe the convergence rate of our algorithms
by estimating the run-time $T$ that could guarantee the minimum value of the norm of received gradients so far is less than a given positive real number $\varepsilon$,
that is,
\[
    \min_{t=0:T-1}\norm{\nabla F(\vx_t)} \le \varepsilon.
\]
\section{Meta-Regularization}
The standard (sub)gradient descent can be derived from the
following minimization problem:
\begin{align}
\label{gd-min}
\vx_{t+1} = \argmin_{\vx \in \gX} ~
\langle\vg_t, \vx - \vx_t\rangle + \frac{1}{2\alpha}\|\vx - \vx_t\|_2^2,
\end{align}
where $\alpha$ is the learning rate. 
To derive our meta-regularization approach, 
we then formulate this minimization problem as a saddle point problem
by adding a meta-regularizer about the difference between the new learning rate $\alpha$ and 
an auxiliary variable $\eta_t$. 
Accordingly, we have 
\begin{align}\label{saddle-point-problem}
    \notag\max_{\alpha \in \gA_t} \min_{\vx \in \gX} & \Psi_t(\vx, \alpha)  \triangleq  
    \langle\vg_t, \vx - \vx_t\rangle \\
   &  + \frac{1}{2} \Big(\frac{1}{\alpha}\norm{\vx - \vx_t} - D(\alpha, \eta_t)\Big),
\end{align} 
where $D(\alpha, \eta)$, a distance function, is defined as our meta-regularizer  and $\gA_t$ is a subset in $\sR$. 
Our framework solves this saddle point problem for a new predictor and a new learning rate.

We usually set the auxiliary variable $\eta_t$ equal to $\alpha_t$, and consider the meta-regularizer as the penalty of the change between $\alpha_{t+1}$ and $\alpha_t$. Sometimes we also can choose the sequence $\{\eta_t\}$ in advance, before our methods start the job. In this case, our framework can be treated as a smoothing technique to stabilize the learning rate. 

\subsection{Update Rules}
\label{sec-update}
In this section we present two update rules of our meta-regularization framework. 
The first update rule is solving saddle point problem (\ref{saddle-point-problem}) exactly. 
That is,
\begin{equation}
    \label{alpha-update1}
    \Psi_t (\vx_{t+1}, \alpha_{t+1}) = \max_{\alpha \in \gA_t}
    \min_{\vx \in \gX} \Psi_t (\vx, \alpha).
\end{equation}
In the setting $\eta_t = \alpha_t$, it is more recommended to employ an alternating strategy in practice. 

The second update rule is an alternatively iterative procedure between $\alpha$ and $\vx$. 
Under the assumption that the optimal value of $\alpha$ is close to $\eta_t$, we solve an approximate equation for finding $\alpha_{t+1}$:
\begin{equation}
    \label{alpha-update2}
    \alpha_{t+1} = \argmax_{\alpha \in \gA_t} 
    \Psi_t \left(\argmin_{\vx \in \gX} \Psi_t(\vx, \alpha_{t}), \alpha \right),
\end{equation}
and update the new predictor $\vx_{t+1}$ via
\[
\vx_{t+1} = \argmin_{\vx \in \gX} \Psi_t(\vx, \alpha_{t+1}).
\]

It is worth noting that these two update rules share  similar performance in some certain situations (see Theorems \ref{algo-12-thm} and \ref{fullbatch-thm} in Section \ref{sec-analysis}). 

\subsection{Diagonal Meta-Regularization}
Consider a generalization of the standard gradient descent \cite{duchi2011adagrad}
\begin{align}
    \label{general-gd}
    \notag \vx_{t+1} &= \Pi_{\gX}^{\diag(\valpha_t)^{1/2}} \Big(\vx_t - \diag(\valpha_t)^{1/2}\vg_t \Big) \\
    \notag &= \argmin_{\vx \in \gX} \left\|\vx_t - \diag(\valpha_t)^{1/2}\vg_t \right\|_{\diag(\valpha_t)^{1/2}}^2 \\
    &= \argmin_{\vx \in \gX} \ \langle\vg_t, \vx {-} \vx_t\rangle + \frac{1}{2}\|\vx {-} \vx_t\|^2_{\diag(\valpha_t)^{-1}}.
\end{align}
Similarly, we can add our meta regularizer to the minimization problem (\ref{general-gd}) as
\begin{align}
    \label{diagnal-saddle}
    \notag\max_{\valpha \in \gA_t} \min_{\vx \in \gX} &\Psi_t(\vx, \valpha) \triangleq 
    \langle\vg_t, \vx - \vx_t\rangle \\
    &+ \frac{1}{2} \Big(\|\vx - \vx_t\|^2_{\diag(\valpha)^{-1}} - D(\valpha, \veta_t)\Big),
\end{align}
where $\gA_t \subseteq \sR_{++}^d$.

\section{Algorithm Design and Analysis}
In this section, we show how to design specific algorithms according to our framework, especially diagonal Meta-Regularization, and provide theoretical analysis for corresponding algorithms. 

\subsection{Algorithms for Two Update Rules}

We choose the $\varphi$-divergence as our meta-regularizer. Accordingly,
we  rewrite the problem (\ref{diagnal-saddle}) as
\begin{align}
    \label{saddle-element-problem}
    \notag &\max_{\valpha \in \gA_t} \min_{\vx \in \gX} \Psi_t(\vx, \valpha) \triangleq
    \sum_{j=1}^d g_{t, j} (x_j - x_{t, j})\\ 
    &+ \frac{1}{2} \left( (x_j - x_{t, j})^2 / \alpha_j - \varphi(\eta_{t,j} / \alpha_j) / \eta_{t, j} \right) .
\end{align}
The form of problem (\ref{saddle-element-problem}) implies that we can solve the problem for each dimension separately, 
and consequently only a little extra run time is required for each step. 
In order to solve the problem feasibly, we always assume that $ \lim_{z\rightarrow +\infty}\varphi'(z)= +\infty $.

The following lemma and Algorithm \ref{algo-1} give the concrete scheme of solving the saddle point problem (\ref{saddle-element-problem}) exactly.
\begin{lemma}
    \label{lemma-update-1}
    Considering problem (\ref{saddle-element-problem}) without constraints and solving the problem exactly, 
    we get new predictor $\vx_{t+1}$ and new learning rate $\valpha_{t+1}$ such that
    \begin{align}
        \label{update-1-eq} 
        \varphi'(\eta_{t,j}/\alpha_{t+1,j}) &= \alpha_{t+1,j}^2 g_{t,j}^2, j = 1, \ldots, d, \\
        \notag \vx_{t+1} &= \vx_t - \valpha_{t+1} \circ \vg_t.
    \end{align}
\end{lemma}

\begin{algorithm}[tb]
\caption{GD with Meta-regularization} \label{algo-1}
\begin{algorithmic}[1]
\REQUIRE $\valpha_0 = \alpha_0 \vone > 0$, $\vx_0$
\FOR{$t = 1$ to $T$}
\STATE Suffer loss $f_t(\vx_t)$;
\STATE Receive subgradient $\vg_t \in \partial f_t(\vx_t)$ of $f_t$ at $\vx_t$;
\STATE Update $\alpha_{t+1, j}$ as the solution of the equation
$\varphi'(\eta_{t,j}/\alpha) = \alpha^2 g_{t,j}^2, j = 1, \ldots, d$;
\STATE Update $\vx_{t+1} = \vx_t - \valpha_{t+1} \circ \vg_t$;
\ENDFOR
\end{algorithmic}
\end{algorithm}

\begin{remark}
    Note that AdaGrad \cite{duchi2011adagrad} and WNGrad \cite{wu2018wngrad} are special cases of Algorithm 1  with a particular choice of $\varphi$ (detailed derivation in Appendix \ref{sec-special}).
    \begin{itemize}
        \item If $\varphi(z) = z + \frac{1}{z} - 2$, then we can derive AdaGrad from Algorithm 1.
        \item If $\varphi(z) = \frac{1}{z} - \log(\frac{1}{z}) - 1$, then we can derive WNGrad from Algorithm 1.
    \end{itemize}
\end{remark}

Applying the  alternating update rule, which we described in Section \ref{sec-update}, under the same assumption in Lemma \ref{lemma-update-1}, 
we obtain the following lemma and Algorithm \ref{algo-2}.
\begin{lemma}
    Considering  problem (\ref{saddle-element-problem}) without constraint 
    and following from the alternating update rule, 
    we get new predictor $\vx_{t+1}$ and new learning rate $\valpha_{t+1}$ as
    \begin{align}
        \label{update-2-eq} \alpha_{t+1,j} &= \frac{\eta_{t,j}}{(\varphi')^{-1}(\eta_{t,j}^2 g_{t,j}^2)}, j = 1, \ldots, d, \\
        \notag \vx_{t+1} &= \vx_t - \valpha_{t+1} \circ \vg_t.
    \end{align}
\end{lemma}

\begin{algorithm}[tb]
\caption{GD with Meta-regularization using alternating update rule} \label{algo-2}
\begin{algorithmic}[1]
\REQUIRE $\valpha_0 = \alpha_0\vone > 0$, $\vx_0$
\FOR{$t = 1$ to $T$}
\STATE Suffer loss $f_t(\vx_t)$;
\STATE Receive $\vg_t \in \partial f_t(\vx_t)$ of $f_t$ at $\vx_t$;
\STATE Update $\alpha_{t+1,j} = \eta_{t,j}/(\varphi')^{-1}(\eta_{t,j}^2 g_{t,j}^2), j = 1, \ldots, d$;
\STATE Update $\vx_{t+1} = \vx_t - \valpha_{t+1} \circ \vg_t$;
\ENDFOR
\end{algorithmic}
\end{algorithm}
Computing the inverse function of $\varphi'$ is usually easier than solving the equation (\ref{update-1-eq}) in practice, 
especially for the widely used $\varphi$-divergences (more details can be found in Appendix \ref{sec-full-version}).

\subsubsection{Full Batch Setting}
\label{sec-algo-12}
Instead of diagonal Meta-Regularization, we consider origin Meta-Regularization (\ref{saddle-point-problem}) here. 
Recall that we set $f_t = F$ in the full batch setting, and assume that $F \in C_L^{1,1}$ without convexity.
In this case, two update rules can be written as \\
\begin{align}
\label{scalar-1}
\begin{cases}
     \varphi'(\alpha_{t} / \alpha_{t+1}) = \alpha_{t+1}^2\norm{\vg_t}, \\
    \vx_{t+1} = \vx_t - \alpha_{t+1} \vg_t.
\end{cases}
\end{align}
\begin{align}
\label{scalar-2}
\begin{cases}
    \alpha_{t+1} = \alpha_t/(\varphi')^{-1}( \alpha_{t}^2\norm{\vg_t}), \\
    \vx_{t+1} = \vx_t - \alpha_{t+1} \vg_t.
\end{cases}
\end{align}
Next we show that convergence of both update rules (\ref{scalar-1}) and (\ref{scalar-2}) are robust to the choice of initial learning rate.

\begin{thm}
    \label{fullbatch-thm}
    Suppose that $\varphi \in C_l^{1, 1} \left( [1, +\infty) \right)$,  $\varphi$ is $\alpha$-strongly convex, 
     $F \in C_L^{1,1}(\sR^d)$, and $F^* = \inf_{\vx}F(\vx) > -\infty$.
    For any $\varepsilon \in (0, 1)$, the sequence $\{\vx_t\}$ obtained from update rules (\ref{scalar-1}) or (\ref{scalar-2}) satisfies
    \begin{align*}
        \min_{j=0:T-1}\norm{\nabla F(\vx_j)} \le \varepsilon,
    \end{align*}
    after $T = \gO\left(\frac{1}{\varepsilon}\right)$ steps.
\end{thm}

More detailed results of Theorem \ref{fullbatch-thm} for runtime can be found in Theorems \ref{scalar-1-thm} and  \ref{scalar-2-thm} in Appendix \ref{sec-full-batch}. 
 Theorem \ref{fullbatch-thm} shows that both runtime of the two update rules can be bound as $\mathcal{O}(1/\varepsilon)$ for any constant $L$ and initial learning rate $\alpha_0$.
Comparing with classical convergence result 
(see (1.2.13) in \cite{nesterov2013introductory} or Theorem \ref{classical-result} in Appendix), 
the upper bound of runtime is $\mathcal{O}(1/\varepsilon)$ 
only for a certain range (related to $L$) of initial learning rates. 

\subsection{Logarithmic Regret Bounds}
\label{sec-log}

In this subsection, we show that  employing some specific distance functions instead of the $\varphi$-divergence as a regularizer can improve convergence rate effectively. We make use of an example of optimization problems in which the objective function is strongly convex.

First, we define $\vmu$-strong convexity.
\begin{defn}[Definition 2.1 in \cite{mukkamala2017scadagrad}]
    Let $\gX \subseteq \sR^d$ be a convex set. 
    We say that a function $f : \gX \rightarrow \sR$ is $\vmu$-strongly convex
    if there exists $\vmu \in \sR^d$ with $\mu_j > 0$ for $j = 1, \cdots, d$ 
    such that for all $\vx, \vy \in \gX$,
    \begin{align*}
        f(\vy) \ge f(\vx) + \langle \nabla f(\vx), \vy -\vx \rangle + \frac{1}{2}\|\vy - \vx\|_{\diag(\vmu)}^2.
    \end{align*}
    Let $\xi = \min_{j=1:d}\mu_j$. Then $f$ is $\xi$-strongly convex (in the usual sense),
    that is, 
    \begin{align*}
        f(\vy) \ge f(\vx) + \langle \nabla f(\vx), \vy -\vx \rangle + \frac{\xi}{2}\|\vy - \vx\|_2^2.
    \end{align*}
\end{defn}

We now propose a modification of Meta-Regularization that we refer to as \emph{SC-Meta-Regularization}. 
The modification uses a family of distance functions $D: \sR_{++}^d \times \sR_{++}^d \to \sR$  as follows
\begin{align}
    \label{new-distance}
    D(\vu, \vv) = \sum_{j=1}^d \varphi(v_j/u_j),
\end{align}
where $\varphi$ is convex function with $\varphi(1) = \varphi'(1) = 0$ like we used in the $\varphi$-divergence.

\begin{remark}
    Same as the $\varphi$-divergence, $D(\vu, \vv) \ge 0$ ~for any $\vu, \vv \in \sR_{++}^d$.
\end{remark}

\begin{figure*}[!htb]
    \centering
    \includegraphics[width=0.8\linewidth]{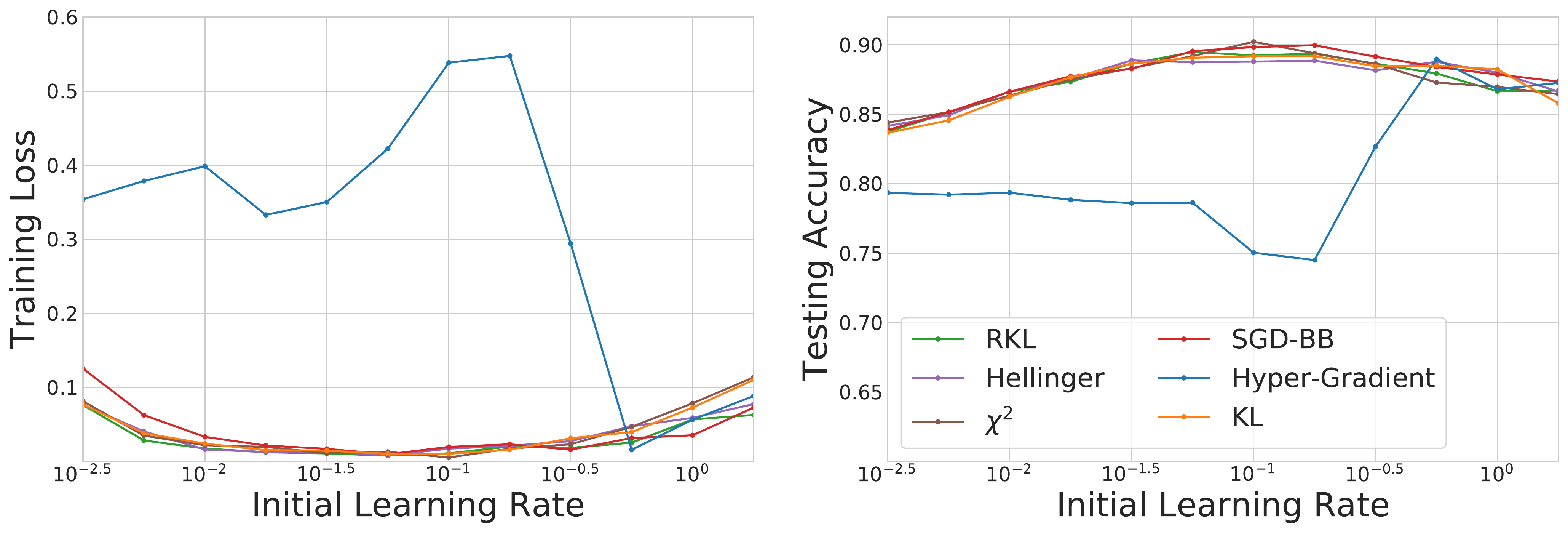}
    \vspace{-0.05in}
    \caption{Convergence performances of algorithms on CIFAR-10 in the online learning setting. \textit{left}: training loss of the last training epoch at different initial learning rates; \textit{right}: testing accuracy of the last training epoch at different initial learning rates.}
    \label{cifar_total}
\end{figure*}

Different from Algorithms \ref{algo-1} and \ref{algo-2}, 
we add a hyper-parameter $\lambda > 0$ like AdaGrad to SC-AdaGrad.
Rewrite problem (\ref{diagnal-saddle}) as
\begin{align}
    \notag\max_{\valpha \in \gA_t} \min_{\vx \in \mathcal{X}} \Psi_t(\vx, \valpha) &\triangleq
    \vg_t^{\top} (\vx - \vx_t) + \frac{1}{2}\|\vx - \vx_t\|_{\diag(\valpha)^{-1}}^2\\ 
    &- \frac{\lambda}{2} \sum_{j=1}^d \varphi(\alpha_{t,j}/\alpha_j),
\end{align}
and give the corresponding algorithm in Algorithm \ref{algo-log}.  
\begin{algorithm}[tb]
\caption{GD with SC-Meta-regularization} \label{algo-log}
\begin{algorithmic}[1]
\REQUIRE $\valpha_0 = \alpha_0\vone > 0$, $\vx_0$
\FOR{$t = 1$ to $T$}
\STATE Suffer loss $f_t(\vx_t)$;
\STATE Receive $\vg_t \in \partial f_t(\vx_t)$ of $f_t$ at $\vx_t$;
\STATE Update $\alpha_{t+1, j}$ as the solution of the equation $\lambda (\alpha_{t,j}/\alpha^2) \varphi' (\alpha_{t,j}/\alpha) = g_{t,j}^2, j = 1, \cdots, d$;
\STATE Update $\vx_{t+1} = \vx_t - \valpha_{t+1}\circ \vg_t $;
\ENDFOR
\end{algorithmic}
\end{algorithm}
\begin{thm}
    \label{algo-log-thm}
    Suppose that $f_t$ is $\vmu$-strongly convex for all $t$, $\varphi \in C_l^{1, 1} \left( [1, +\infty) \right)$, and $\varphi$ is $\gamma$-strongly convex.
    Assume that $\|\vg_t\|_{\infty} \le G$, and $\lambda \ge G^2 / (\gamma\min_{j=1:d}\mu_j)$.
    Then the sequence $\{\vx_t\}$ obtained from Algorithm \ref{algo-log} satisfies
    \begin{align*}
        2R(T) &\le l \left(1 + \frac{\alpha_0G^2}{\lambda l}\right)^2 \sum_{j=1}^d 
        \ln \left( 1 + \frac{\alpha_0\norm{g_{0:T-1,j}}}{\lambda l} \right)\\
        &+ \|\vx_0 - \vx^*\|_2^2/\alpha_0.
    \end{align*}
\end{thm}

Under the assumption in Theorem \ref{algo-log-thm}, we note that $\norm{g_{0:T-1,j}} \le G^2 T$. Hence, $R(T) = \gO(\ln(T))$ holds.
\subsection{Theoretical Analysis}
\label{sec-analysis}
In this subsection, we always set $\veta_t = \valpha_t$ and assume that $\vx$ and $\valpha$ are unconstrained, 
i.e., $\gX = \sR^d$ and $\gA_t = \sR_{++}^d$.
We first demonstrate the monotonicity of both the two update rules from Algorithm \ref{algo-1} and \ref{algo-2} in Section \ref{sec-algo-66}. Afterwards, 
we discuss the convergence rate of the two update rules in online convex learning setting in Section \ref{sec-algo-34}
and establish a theorem about the regret bounds in Section \ref{sec-algo-34}.
Furthermore, we turn to full batch setting 
with assumption that the objective function $F$ is $L$-smooth but not necessarily convex in Section \ref{sec-algo-12}.
Our results for both the settings show that the convergence of our algorithms are robust to the choice of initial learning rates and do not rely on the Lipschitz constant or smoothness constant.

\subsubsection{Monotonicity}
\label{sec-algo-66}
We point out the monotonicity of learning rate sequences $\{\valpha_t\}$ in our algorithms (proof can be found in Appendix \ref{sec-Monotonicity}).
\begin{lemma}
    \label{lemma-monotonicity}
    The sequences $\{\valpha_{t}\}$ obtained from Algorithm \ref{algo-1} or \ref{algo-2} satisfies
    $\valpha_{t+1} \le \valpha_t$.
\end{lemma}
This phenomenon is common in general training setting like learning rate decay and necessary in several convergence proof including online learning \cite{dekel2012optimal, ghadimi2013stochastic} and classical convex optimization \cite{bubeck2015convex} .
\subsubsection{Online Learning Setting}
\label{sec-algo-34}
We now establish the result of regrets of Algorithms \ref{algo-1} and  \ref{algo-2} in online convex learning, 
i.e., the $f_t$ are convex.
Exactly, we try to bound regrets (\ref{regret})
by $\gO(\sqrt{T})$ for Algorithms \ref{algo-1} and \ref{algo-2}. 
In other words, if $f_t = f$ are the same function, we get a $\gO(1/\sqrt{T})$ convergent rate.
\begin{thm}
    \label{algo-12-thm}
    Suppose that $\varphi \in C_l^{1, 1} \left( [1, +\infty) \right)$, and $\varphi$ is $\gamma$-strongly convex.
    Assume that $\|\vg_t\|_{\infty} \le G$, $\|\vx_t - \vx^*\|_{\infty} \le D_{\infty}$.
    Then the sequence $\{\vx_t\}$ obtained from Algorithm \ref{algo-1} satisfies
    \begin{align*}
        2R(T) &\le \left(1 + \frac{D_{\infty}^2}{\gamma}\right)\sqrt{2l + 4\alpha_0^2 G^2} \sum_{j=1}^d \|g_{0:T-1,j}\|_2 \\
        &+ \|\vx_0 - \vx^*\|_2^2 / \alpha_0,
    \end{align*}
    and the sequence $\{\vx_t\}$ obtained from Algorithm \ref{algo-2} 
    satisfies
    \begin{align*}
        2R(T) &\le \left(1 + \frac{D_{\infty}^2}{\gamma}\right) \max\left\{\sqrt{2l}, 2\alpha_0 G\right\} 
        \sum_{j=1}^d \|g_{0:T-1,j}\|_2 \\
        &+ \|\vx_0 - \vx^*\|_2^2 / \alpha_0.
    \end{align*}
\end{thm}

Note that under the assumption in Theorem \ref{algo-12-thm}, 
$\sum_{j=1}^d \|g_{0:T-1,j}\|_2 \le d G \sqrt{T}$, 
hence $R(T) = \mathcal{O}(\sqrt{T})$.
Our result is comparable to the best known bound for convex online learning problem 
\cite{duchi2011adagrad, kingma2015adam}. 

We provide a proof sketch here and more detailed proof can be found in Appendix \ref{sec-regret}.
\begin{proof}[proof sketch]
Following from $\vx_{t+1} = \vx_t - \diag(\valpha_{t+1})\vg_t$, we can get 
\begin{align*}
    &2R(T) = 2\sum_{t=0}^{T-1}(f_t(\vx_t)-f_t(\vx_*)) 
    \le 2\sum_{t=0}^{T-1} \vg_{t}^{\top}(\vx_t - \vx^*) \\
    &= \sum_{t=0}^{T-1} \left(\|\vx_{t} - \vx^*\|_{B_{t+1}}^2
    - \|\vx_{t+1} - \vx^*\|_{B_t}^2 + \|\vg_t\|_{B_{t+1}^{-1}}^2 \right) \\
    &\le \sum_{t=0}^{T-1} \left(\|\vx_{t} - \vx^*\|_{(B_{t+1} - B_t)}^2 + \|\vg_t\|_{B_{t+1}^{-1}}^2\right) + \beta_0\|\vx_0 - \vx^*\|_2^2 \\
    &\le \sum_{t=0}^{T-1} \left(\|\vx_{t} - \vx^*\|_{\infty}^2 \|\vbeta_{t+1} - \vbeta_{t}\|_1 + \|\vg_t\|_{B_{t+1}^{-1}}^2 \right) \\
    &\qquad\qquad+ \beta_0\|\vx_0 - \vx^*\|_2^2 \\ 
    &\le D_{\infty}^2 \sum_{t=0}^{T-1} \sum_{j=1}^d (\beta_{t+1,j} - \beta_{t,j})
    + \sum_{t=0}^{T-1} \sum_{j=1}^d \frac{g_{t,j}^2}{\beta_{t+1,j}} \\
    &\qquad\qquad+ \beta_0 \|\vx_0 - \vx^*\|_2^2,
\end{align*}
where $\vbeta_t = 1 / \valpha_t$, and $B_t = \diag(\vbeta_t)$. \\

For Algorithm \ref{algo-1}, 
\begin{align*}
    &\sum_{t=0}^{T-1} (\beta_{t+1,j} - \beta_{t,j})
    \le \frac{1}{\gamma} \sum_{t=0}^{T-1} \frac{g_{t,j}^2}{\beta_{t+1,j}}, \\
    &\sum_{t=0}^{T-1} \frac{g_{t,j}^2}{\beta_{t+1,j}}
    \le \frac{\sqrt{2l\beta_{0}^2 + 4G^2}}{\beta_{0}} \sqrt{\sum_{i=0}^{T-1}g_{t,j}^2} \\
    &\qquad \qquad ~~~ = \frac{\sqrt{2l\beta_{0}^2 + 4G^2}}{\beta_{0}} \|g_{0:T-1,j}\|_2.
\end{align*}
Thus 
\begin{align*}
    2R(T) &\le \left(1 + \frac{D_{\infty}^2}{\gamma}\right)\frac{\sqrt{2l\beta_0^2 + 4G^2}}{\beta_0} \sum_{j=1}^d \|g_{0:T-1,j}\|_2 \\
        &+ \beta_0 \|\vx_0 - \vx^*\|_2^2 \\
        &=\left(1 + \frac{D_{\infty}^2}{\gamma}\right)\sqrt{2l + 4\alpha_0^2G^2} \sum_{j=1}^d \|g_{0:T-1,j}\|_2 \\
        &+ \|\vx_0 - \vx^*\|_2^2 / \alpha_0.
\end{align*}
Similarly, for Algorithm \ref{algo-2}, 
\begin{align*}
    &\sum_{t=1}^{T-1} (\beta_{t,j} - \beta_{t-1,j})
    \le \frac{1}{\gamma} \sum_{t=0}^{T-1} \frac{g_{t,j}^2}{\beta_{t,j}},\\
    &\sum_{j=1}^d \sum_{t=0}^{T-1} \frac{g_{t,j}^2}{\beta_{t+1,j}} 
    \le \sum_{j=1}^d \sum_{t=0}^{T-1} \frac{g_{t,j}^2}{\beta_{t,j}} \\
    &\le \max\left\{\sqrt{2l}, \frac{2G}{\beta_0}\right\} \sum_{j=1}^d \|g_{0:T-1,j}\|_2.
\end{align*}
Therefore,
\begin{align*}
    2R(T) &\le \left(1 + \frac{D_{\infty}^2}{\gamma}\right) 
        \max\left\{\sqrt{2l}, \frac{2G}{\beta_0}\right\} \sum_{j=1}^d \|g_{0:T-1,j}\|_2 \\
        &+ \beta_0 \|\vx_0 - \vx^*\|_2^2 \\
        &= \left(1 + \frac{D_{\infty}^2}{\gamma}\right) 
        \max\left\{\sqrt{2l}, 2\alpha_0G\right\} \sum_{j=1}^d \|g_{0:T-1,j}\|_2 \\
        &+ \|\vx_0 - \vx^*\|_2^2 / \alpha_0.
\end{align*}
\end{proof}

\section{Numerical Experiments}

In this paper our principal focus has been to develop a novel approach to adaptively choosing the learning rate during the training process.
It would be also interesting to empirically compare our approach with the BB method and Hyper-Gradient Descent in both full batch and online learning settings.
Considering the large amount of valid regularization terms, the term is constrained to be generated from the $\varphi$-divergence in the following numerical experiments.
For both simplicity and generalization, we merely utilize several common $\varphi$-divergences to derive algorithms, without any delicate design.
Experimental results have revealed that these algorithms obtain comparable performance, and even outperform the BB method and Hyper-Gradient Descent in some cases.

\subsection{The Set-Up}

In the experiments,
four common $\varphi$-divergences are used to derive the representative algorithms in \textit{Meta-Regularization} framework (full implementations are displayed in the Appendix \ref{sec-full-version}):
\begin{itemize}
\setlength{\itemsep}{0pt}
\setlength{\parsep}{0pt}
    \item $KL(t)=t \log t -t +1$ leads to KL algorithm.
    \item $RKL(t)=-\log t + t - 1$ leads to RKL algorithm.
    \item $Hellinger(t)=(\sqrt{t}-1)^2$ leads to H algorithm.
    \item $\chi^2(t)=(t-1)^2$ leads to $\chi^2$ algorithm.
\end{itemize}
With any chosen $\varphi$-divergence described above, the corresponding algorithm adopts the update rule described in Algorithm \ref{algo-2} rather than in Algorithm \ref{algo-1}.
This mainly comes out of the consideration on computation effectiveness (detailed explanations are displayed in Appendix \ref{sec-full-version}).

To maintain stable performance, the technique of growth clipping is applied to all algorithms in our framework.
Actually, growth clipping fulfills the constraints placed on the shrinking speed of the learning rate, which we fully explain in Appendix \ref{maxmin}.
Specifically, after each update, the updated learning rate can not be smaller than half of the original learning rate.

\begin{figure*}
    \centering
    \includegraphics[width=0.755\linewidth]{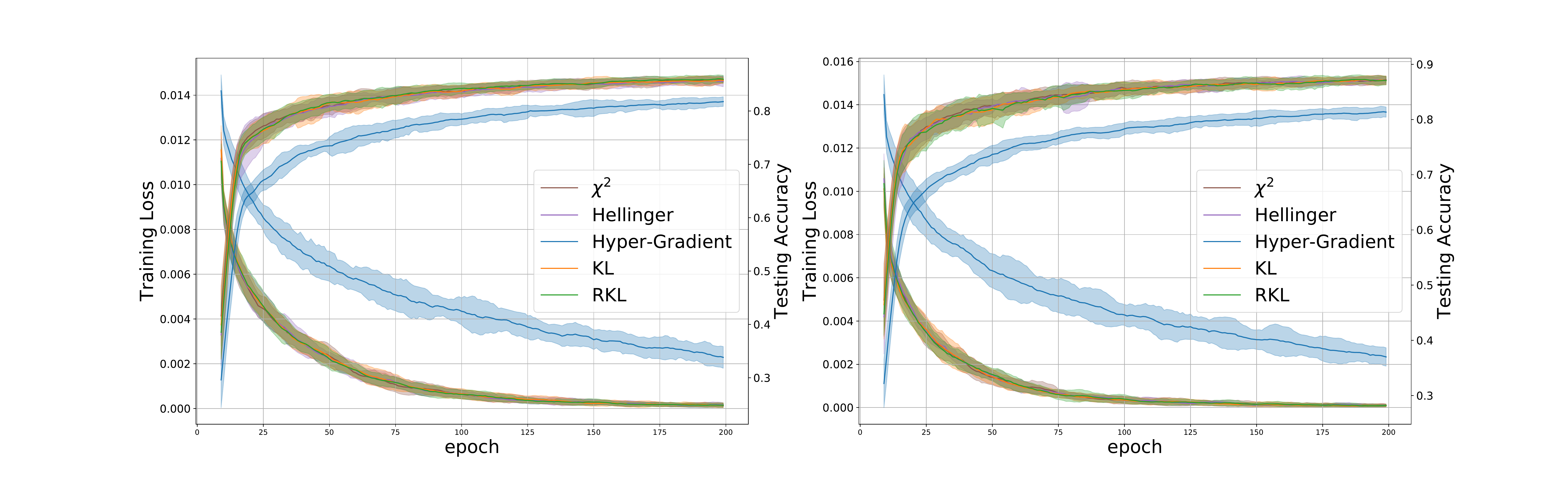}
    \vspace{-0.05in}
    \caption{Training process in terms of training loss and testing accuracy on different algorithms with different initial learning rate (\textit{left}: 0.005; \textit{right}: 0.01. We repeat our experiments for three times in each curve with different random seeds, and plot shadow error region with two times standard error.}
    \label{cifar_detail}
\end{figure*}

\begin{figure*}
    \centering
    \includegraphics[width=0.755\linewidth]{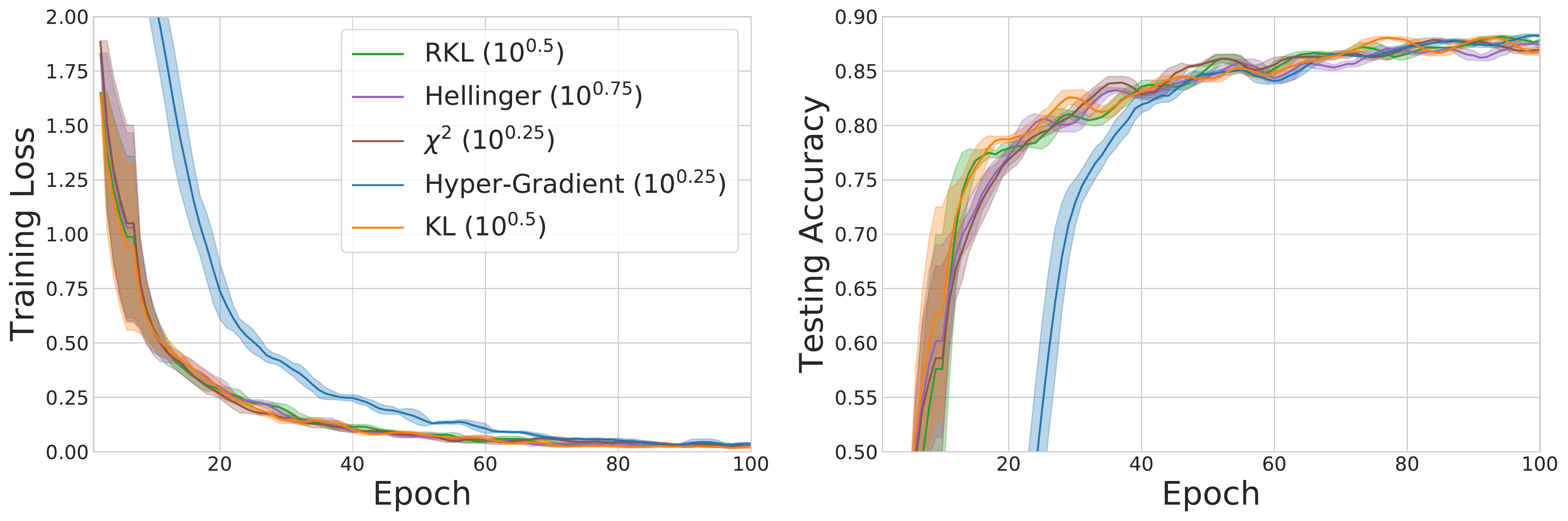}
    \vspace{-0.05in}
    \caption{Training process at initial learning rate with least training loss on different algorithms(\textit{left}: training loss at each training epoch; \textit{right}: testing accuracy at each training epoch). We repeat our experiments for three times in each curve with different random seeds, and plot shadow error region with two times standard error.}
    \label{cifar_best}
\end{figure*}

Numerical experiments involve the above four proposed algorithms as well as the BB method, and Hyper-Gradient Descent algorithms.
These algorithms are evaluated on tasks of image classification with a logistic classifier on the databases of MNIST \cite{lecun2010mnist} and CIFAR-10 \cite{krizhevsky2009learning}.
Experiments are run using Tensorflow \cite{abadi2016tensorflow}, on a machine with Intel Xeon E5-2680 v4 CPU, 128 GB RAM, and NVIDIA Titan Xp GPU.

\subsection{Full Batch Setting}
We investigate our algorithms in the full batch setting on the MNIST database where algorithms receive the exact gradients of the objective loss function each iteration.
The network used in the classifier merely consists of one fully connected layer.
The train loss of different algorithms after 50 epochs of training is displayed in Figure \ref{fullbatch}.

\begin{figure}[H]
    \centering
    \includegraphics[width=0.9\linewidth]{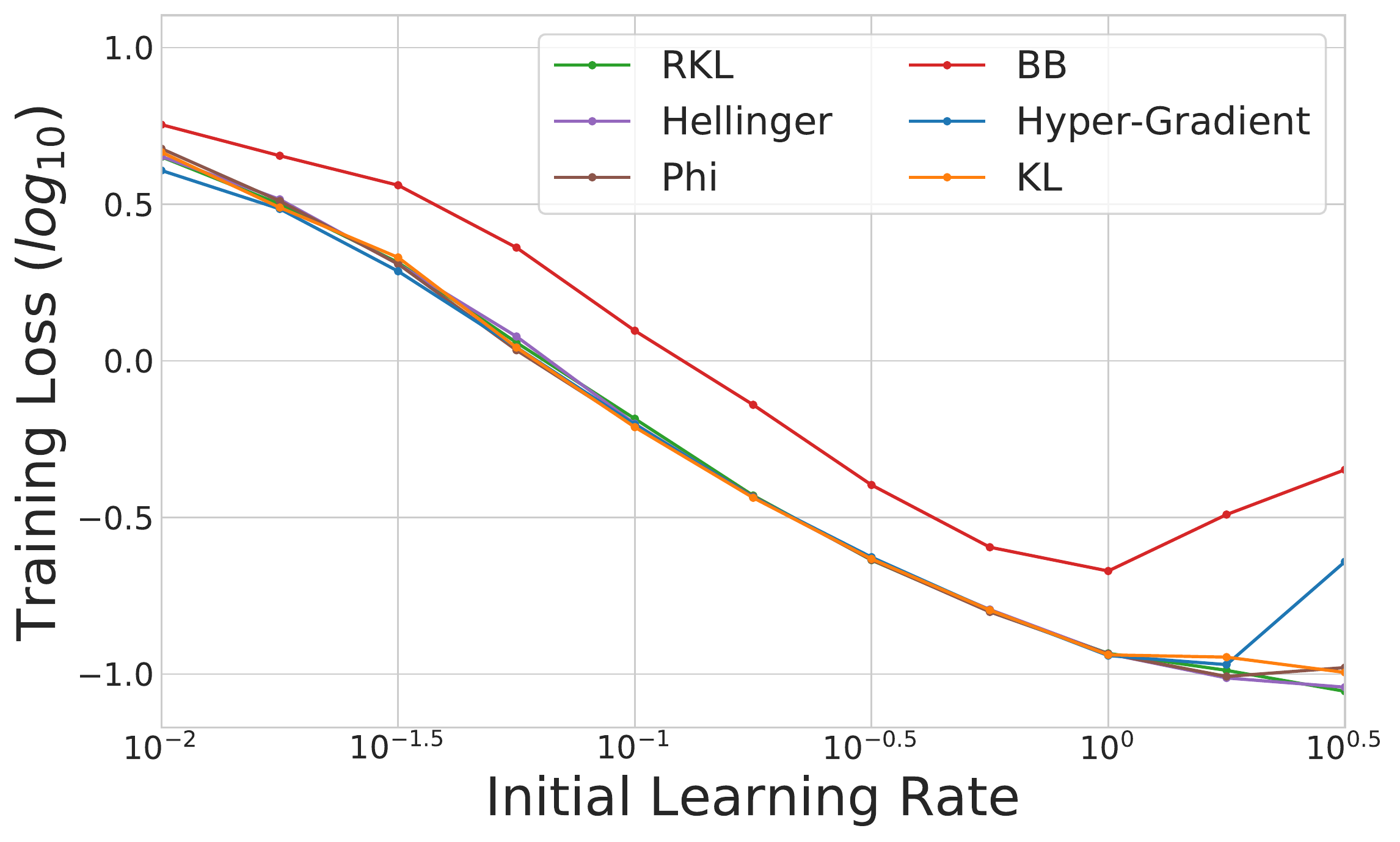}
    \vspace{-0.1in}
    \caption{The training of the last training epoch on MNIST at different initial learning rates in full batch setting.}
    \label{fullbatch}
\end{figure}

All of the four algorithms derived from our framework are shown to obtain comparable performance with Hyper-Gradient Descent, regardless of the initial learning rate.
Moreover, the performance of the BB method is congruously inferior to that of the algorithms from \textit{Meta-Regularization}.
Such advantage comes more obvious while the initial learning rate goes larger.

\subsection{Online Learning Setting}

In the online learning setting, we train a VGG Net \cite{simonyan2014very} with batch normalization on the CIFAR-10 database with a batch size of 128, and an $\ell_2$ regularization coefficient of $10^{-4}$. 
We as well perform data augmentation as \cite{he2016deep} to improve the training.
The train loss as well as test accuracy of different algorithms at different initial learning rates after 100 epochs of training are displayed in Figure \ref{cifar_total}.

All of the four algorithms based on \textit{Meta-Regularization} are shown to obtain comparable performance with the BB methods, exhibiting a relatively low training loss within a large range of initial learning rates.
Besides, the advantages of these four algorithms over Hyper-Gradient Descent are obvious in the following two aspects: a generally better convergence performance and a faster convergence speed.
From Figure \ref{cifar_total}, it is apparent that Hyper-Gradient fails to maintain either a low training loss or a high testing accuracy while the initial learning rate ranging from $10^{-2.5}$ to $10^{-0.5}$.
Specifically, Figure \ref{cifar_detail} displays the training process at several given learning rates, which conforms to the above observation. 
For a fair comparison of convergence speed, the initial learning rates with least training loss are respectively fixed for involved algorithms.
In  Figure \ref{cifar_best}, it is obviously observed that the algorithms from \textit{Meta-Regularization} obtain a comparable convergence performance but a faster convergence speed than Hyper-Gradient Descent, in terms of both training loss and testing accuracy.

\bibliographystyle{plain}
\bibliography{references}
\newpage
\section{Solution Existence}
Note that the function $h(1/\alpha) = 1/\alpha^2 \varphi'(\eta_{t,j}/\alpha)$ is an increasing continuous function  and $\lim_{z \rightarrow + \infty}\varphi'(z) = +\infty$ $\varphi'(1) = 0$, so $[0, +\infty)$ is a subset of the range of $h(1/\alpha)$ and the solution of (\ref{update-1-eq}) exists. \\
For the same reason, the solution of (\ref{update-2-eq}) exists.

\section{Special Cases of Algorithm 1}
\label{sec-special}
In this section, We will point out that Adagrad \cite{duchi2011adagrad} and WNGrad \cite{wu2018wngrad} are special cases of Algorithm 1.\\
If we set $\varphi(z) = z + \frac{1}{z} - 2$,
then the new learning rate $1/\valpha_{t+1}$ can be obtained by
\begin{align*}
    \frac{1}{\alpha_{t+1,j}^2}\left(1 - \frac{\alpha_{t+1,j}^2}{\alpha_{t,j}^2}\right) = g_{t,j}^2, \ j = 1, \cdots, d,
\end{align*}
that implies, 
\begin{align*}
    \frac{1}{\valpha_{t+1}^2} = \frac{1}{\valpha_{t}^2} + \vg_{t}^2,
\end{align*}
and we drive AdaGrad from Meta-Regularization.

Similarly, we can get WNGrad by setting $\varphi(z) = \frac{1}{z} - \log \frac{1}{z} - 1$.
In fact, $1/\valpha_{t+1}$ employs update
\begin{align*}
    \frac{1}{\alpha_{t+1,j}^2}\left(\frac{1/\alpha_{t,j}(1/\alpha_{t+1,j} - 1/\alpha_{t,j})}{1/\alpha_{t+1,j}^2}\right) = g_{t,j}^2, \ j = 1, \cdots, d,
\end{align*}
on the other words,
\begin{align*}
    \frac{1}{\valpha_{t+1}} = \frac{1}{\valpha_{t}} + \valpha_t\vg_{t}^2,
\end{align*}
i.e., the update rule of WNGrad.

\section{Max-min or min-max}
\label{maxmin}
\begin{lemma}
    \label{lemma-8}
Suppose that $\mathcal{A}_t = [b_{t,1}, B_{t,1}] \times \cdots \times [b_{t,d}, B_{t,d}]$, and $\mathcal{X} = \sR^d$. 
    Let $\valpha^*$ be the solution of unconstrained problem $\max_{\valpha} (\min_{\vx} \Psi_t(\vx, \valpha))$. 
    Then the solution of problem $\max_{\valpha \in \mathcal{B}_t} (\min_{\vx} \Psi_t(\vx, \valpha))$ is
    \begin{align*}
        \alpha_j = \min \{ \max \{\alpha_j^*, b_{t,j}\}, B_{t,j} \}, \text{ for } j = 1, \cdots, d.
    \end{align*}
\end{lemma}

\begin{proof}
    First, it is trivial to get
    \begin{align*}
        \Psi_{t, \vx}(\valpha) \triangleq \ & min_{\vx} \Psi_t(\vx, \valpha) =
        \Psi_t\left(\vx_t - \valpha\circ\vg_t, \valpha \right) \\
        = \ & -\frac{1}{2}\|\vg_t\|_{\diag(\valpha)}^2 - \frac{1}{2}D_\varphi(\valpha, \veta_t) \\
        = \ & -\frac{1}{2}\sum_{j=1}^d \left( \alpha_jg_{t,j}^2 + \frac{1}{\eta_{t,j}}\varphi\left(\frac{\eta_{t,j}}{\alpha_j}\right) \right).
    \end{align*}
    The partial derivative of $\Psi_{t, \vx}(\valpha)$  with respect to $\alpha_j$ is
    \begin{align*}
        \frac{\partial \Psi_{t, \vx}(\valpha)}{\partial \alpha_j} =
        -\frac{1}{2} \left(g_{t,j}^2 - \frac{1}{\alpha_j^2}\varphi'\left(\frac{\eta_{t,j}}{\alpha_j}\right) \right).
    \end{align*}
    Note that $\varphi$ is a convex function, so $\varphi'$ is a non-decreasing function, 
    and $\frac{\partial \Psi_{t, \vx}(\valpha)}{\partial \alpha_j}$ is a non-increasing function. 
    Recall that $\valpha^*$ be the solution of unconstrained problem $\max_{\valpha} (\min_{\vx} \Psi_t(\vx, \valpha))$,
    hence, $\alpha_j^*$ is a zero of function $\frac{\partial \Psi_{t, \vx}(\valpha)}{\partial \alpha_j}$. \\
    Moreover, if $\alpha_j^* > B_{t,j}$, we have $\frac{\partial \Psi_{t, \vx}(\valpha)}{\partial \alpha_j} \ge 0$.
    Thus, $\Psi_{t, \vx}(\valpha)$ with respect to $\alpha_j$ is a non-increasing function,
    and $\argmax_{\alpha_j} \Psi_{t, \vx}(\valpha) = B_{t,j}$.
    For a similar reason, if $\alpha_j^* < b_{t,j}$, then $\argmax_{\alpha_j} \Psi_{t, \vx}(\valpha) = b_{t,j}$.
    In conclusion, 
    \begin{align*}
        \argmax_{\alpha_j\in [b_{t,j }, B_{t,j}]} \Psi_{t, \vx}(\valpha) = \min \{ \max \{\alpha_j^*, b_{t,j}\}, B_{t,j} \}, \text{ for } j = 1, \cdots, d.
    \end{align*}
\end{proof}

\section{Monotonicity}
\label{sec-Monotonicity}
In this section, We provide the proof of Lemma \ref{lemma-monotonicity}.
Denote that $\Psi_{t, \vx} (\alpha) = \min_{\vx \in \mathcal{X}} \Psi_t(\vx, \alpha)$.
\begin{lemma}
    $\alpha_{t+1}$ obtained from equation (\ref{update-1-eq}) satisfies
    $\alpha_{t+1} \le \eta_t$.
\end{lemma}
\begin{proof}
    Recall that $\varphi(1) = \varphi'(1) = 0$, so $\varphi(x) \ge 0$ for all $x$
    and $D_{\varphi}(\alpha, \eta_t) = \varphi(\eta_t/\alpha)/\eta_t \ge 0$.
    If $\alpha > \eta_t$, then for all $\vx \in \mathcal{X}$
    \begin{align*}
        \Psi_{t} (\vx, \alpha) &=
        \vg_t^{\top}(\vx - \vx_t) + \frac{1}{2\alpha}\norm{\vx - \vx_t} - \frac{1}{2}D_{\varphi}(\alpha, \eta_t) \\
        &< \vg_t^{\top}(\vx - \vx_t) + \frac{1}{2\eta_t}\norm{\vx - \vx_t} \\
        &= \Psi_{t} (\vx, \eta_t).
    \end{align*}
    Hence, $\min_{\vx \in \mathcal{X}} \Psi_t (\vx, \alpha) < \min_{\vx \in \mathcal{X}} \Psi_t (\vx, \eta_t)$,
    i.e., $\Psi_{t, \vx} (\alpha) < \Psi_{t, \vx} (\eta_t)$. \\
    It means $\alpha_{t+1} = \argmax_{\alpha \in \mathcal{A}} \Psi_{t, \vx}(\alpha) \le \eta_t$.
\end{proof}

\begin{lemma}
    $\alpha_{t+1}$ obtained from equation (\ref{update-2-eq}) satisfies
    $\alpha_{t+1} \le \eta_t$.
\end{lemma}
\begin{proof}
    Let $\vy = \argmin_{\vx}\Psi(\vx, \veta_t)$.
    If $\alpha > \eta_t$, then
    \begin{align*}
        \Psi_{t} (\vy, \alpha) &=
        \vg_t^{\top}(\vy - \vx_t) + \frac{1}{2\alpha}\norm{\vy - \vx_t}
        - \frac{1}{2}D_{\varphi}(\alpha, \eta_t) \\
        &< \vg_t^{\top}(\vy - \vx_t) + \frac{1}{2\eta_t}\norm{\vy - \vx_t} \\
        &= \Psi_{t} (\vy, \eta_t).
    \end{align*}
    Hence $\alpha_{t+1} = \argmax_{\alpha \in \mathcal{A}}\Psi_t (\vy, \alpha) \le \eta_t$.
\end{proof}

\section{Regrets in online learning setting}
\label{sec-regret}
Recall the definition of regret
\begin{align}
    R(T) = \sum_{t=0}^{T-1}(f_t(\vx_t) - f_t(\vx^*)),
\end{align}
where $\vx^* = \argmin_{\vx \in \mathcal{X}}\sum_{t=0}^{T-1} f_t(\vx)$.
We show our Algorithm \ref{algo-1}, \ref{algo-2} derived from meta-regularization have $\mathcal{O} (\sqrt{T})$ regret bounds.


\begin{lemma}
    \label{regret-basic-lemma}
    Consider an arbitrary real-valued sequence \{$a_i$\} and its vector representation
    $a_{1:i} = (a_1, \cdots, a_i)^{\top}$. Then
    \begin{align}
        \label{regret-basic-lemma-eq}
        \sum_{t=1}^{T} \frac{a_t^2}{\|a_{1:t}\|_2} \le 2 \|a_{1:T}\|_2
    \end{align}
    holds.
\end{lemma}
\begin{proof}
    Let us use induction on $T$ to prove inequality (\ref{regret-basic-lemma}).
    For $T = 1$, the inequality trivially holds.
    Assume the bound (\ref{regret-basic-lemma-eq}) holds true for $T - 1$, in which case
    \begin{align*}
        \sum_{t=1}^{T} \frac{a_t^2}{\|a_{1:t}\|_2} \le 2\|a_{1:T-1}\|_2 + \frac{a_T^2}{\|a_{1:T}\|_2}.
    \end{align*}
    We denote $b_T = \sum_{t=1}^{T}a_t^2$ and have
    \begin{align*}
        2\|a_{1:T-1}\|_2 + \frac{a_T^2}{\|a_{1:T}\|_2}
        &= 2\sqrt{b_T - a_T^2} + \frac{a_T^2}{\sqrt{b_T}} \\
        &\le 2\sqrt{b_T - a_T^2 + \frac{a_T^4}{4b_T}} + \frac{a_T^2}{\sqrt{b_T}} \\
        &= 2\sqrt{b_T}.
    \end{align*}
\end{proof}

\begin{lemma}
    \label{iteration-lemma}
    Suppose the sequence $\{\vx_t\}$ and sequence $\{\valpha_t\}$ satisfy $\vx_{t+1} = \vx_t - \valpha_{t+1} \circ \vg_t$.
    Then the regret satisfies
    \begin{align*}
        2R(T) \le \sum_{t=0}^{T-1} \|\vg_t\|_{\diag(\valpha_{t+1})}^2
        + \sum_{t=0}^{T-1} \|\vx_{t} - \vx^*\|_{\diag(\valpha_{t+1} - \valpha_{t})^{-1}}^2
        + \|\vx_0 - \vx^*\|_{\diag(\valpha_0)^{-1}}^2
    \end{align*}
\end{lemma}
\begin{proof}
    Note that
    \begin{align*}
        \vx_{t+1} = \vx_t - \diag(\valpha_{t+1})\vg_t,
    \end{align*}
    and
    \begin{align*}
        \quad&\ \|\vx_{t+1} - \vx^*\|_{\diag(\valpha_{t+1})^{-1}}^2 \\
        =&\ \|\vx_{t} - \vx^* - \diag(\valpha_{t+1})\vg_t\|_{\diag(\valpha_{t+1})^{-1}}^2 \\
        =&\ \|\vx_{t} - \vx^*\|_{\diag(\valpha_{t+1})^{-1}}^2 + \|\vg_t\|_{\diag(\valpha_{t+1})}^2
        - 2 \vg_{t}^{\top}(\vx_t - \vx^*),
    \end{align*}
    i.e.,
    \begin{align}
        2 \vg_t^{\top} (\vx_t - \vx^*) = \|\vg_t\|_{\diag(\valpha_{t+1})}^2
        + \left(\|\vx_{t} - \vx^*\|_{\diag(\valpha_{t+1})^{-1}}^2
        - \|\vx_{t+1} - \vx^*\|_{\diag(\valpha_{t+1})^{-1}}^2\right).
    \end{align}
    Hence
    \begin{align*}
        2R(T) &= 2\sum_{t=0}^{T-1}(f_t(\vx_t)-f_t(\vx_*)) \\
        &\le 2\sum_{t=0}^{T-1} \vg_{t}^{\top}(\vx_t - \vx^*) \\
        &= \sum_{t=0}^{T-1} \|\vg_t\|_{\diag(\valpha_{t+1})}^2
        + \sum_{t=0}^{T-1} \left(\|\vx_{t} - \vx^*\|_{\diag(\valpha_{t+1})^{-1}}^2
        - \|\vx_{t+1} - \vx^*\|_{\diag(\valpha_{t+1})^{-1}}^2\right) \\
        &\le \sum_{t=0}^{T-1} \|\vg_t\|_{\diag(\valpha_{t+1})}^2
        + \sum_{t=0}^{T-1} \|\vx_{t} - \vx^*\|_{\diag(\valpha_{t+1} - \valpha_{t})^{-1}}^2
        + \|\vx_0 - \vx^*\|_{\diag(\valpha_0)^{-1}}^2.
    \end{align*}
\end{proof}

\begin{lemma}
    \label{algo-1-bound-lemma}
    Suppose an increasing function $\psi$ satisfies $\psi(1) = 0$ and $\psi(x) \le l(x - 1)$.
    Consider a real valued sequence $\{g_t\}_{t=0:T-1}$ and a positive sequence $\{\beta_t\}_{t=0:T}$ which satisfies
    $\vert g_t \vert \le G$, $\beta_{t+1}^2 \psi\left(\frac{\beta_{t+1}}{\beta_t}\right) = g_t^2$, $t=0,\cdots,T-1$, $\beta_0 \ge 0$.
    We can bound $\beta_T$ as
    \begin{align}
        \beta_{t} \ge c \sqrt{\beta_0^2 + \frac{2}{l}\sum_{i=0}^{t-1}g_i^2}, \ t = 1, \cdots, T
    \end{align}
    where $c = \sqrt{\frac{\beta_0^2}{\beta_0^2 + 2G^2/l}}$.
    Moreover, we have
    \begin{align}
        \sum_{t=0}^{T-1} \frac{g_t^2}{\beta_{t+1}}
        \le \frac{\sqrt{2l\beta_0^2 + 4G^2}}{\beta_0} \sqrt{\sum_{t=0}^{T-1}g_t^2}.
    \end{align}
\end{lemma}
\begin{remark}
    We point out that
    \begin{itemize}
        \item $\beta_{t+1} \ge \beta_t$
        (If $\beta_{t+1} < \beta_t$, then $\beta_{t+1}^2\psi(\beta_{t+1}/\beta_t) < 0 \le g_t^2$),
        \item $\beta_{t+1}$ is unique with respect to $\beta_t$ due to the fact that the
        function $\hat{\psi}(\beta) = \beta^2 \psi(\beta/\beta_t)$ is strictly increasing.
    \end{itemize}
\end{remark}
\begin{proof}
    Assume that $\beta_{t} \ge c \sqrt{\beta_0^2 + \frac{2}{l}\sum_{i=0}^{t-1}g_i^2}$,
    where $c > 0$ is a variable coefficient. \\
    Let us find out a specific $c$ such that
    $\beta_{t+1} \ge c \sqrt{\beta_0^2 + \frac{2}{l}\sum_{i=0}^{t}g_i^2}$.\\
    Note that
    \begin{align}
        \label{algo-3-eq-1}
        g_t^2 = \beta_{t+1}^2 \psi\left(\frac{\beta_{t+1}}{\beta_t}\right) \le l \beta_{t+1}^2 \left(\frac{\beta_{t+1}}{\beta_t} - 1\right).
    \end{align}
    Define a cubic polynomial
    \begin{align*}
        h(\beta) = \frac{l}{\beta_t} \beta^3 - l \beta^2 - g_t^2,
    \end{align*}
    and $h$ is an increasing function when $\beta \ge \beta_t$. \\
    If $h\left(c \sqrt{\beta_0^2 + \frac{2}{l}\sum_{i=0}^{t}g_i^2}\right) \le 0$,
    according to $h(\beta_{t+1}) \ge 0$,
    then $\beta_{t+1} \ge c \sqrt{\beta_0^2 + \frac{2}{l}\sum_{i=0}^{t}g_i^2}$.\\
    Denote $b = \beta_0^2 + \frac{2}{l}\sum_{i=0}^{t-1}g_i^2$.
    So we just need to choose $c$ such that
    \begin{align*}
        h\left(c \sqrt{\beta_0^2 + \frac{2}{l}\sum_{i=0}^{t}g_i^2}\right)
        \le lc^2 (b + 2g_t^2/l)\left(\frac{\sqrt{b + 2g_t^2/l}}{\sqrt{b}} - 1\right) - g_t^2
        \le 0,
    \end{align*}
    where the first inequality holds for the assumption $\beta_{t} \ge c \sqrt{\beta_0^2 + \frac{2}{l}\sum_{i=0}^{t-1}g_i^2}$, or
    \begin{align*}
        \frac{c^2}{\sqrt{b}}(b + 2g_t^2/l)\frac{2g_t^2/l}{\sqrt{b + 2g_t^2/l} + \sqrt{b}} \le g_t^2/l,
    \end{align*}
    or
    \begin{align*}
        \frac{2c^2}{\sqrt{b}}(b + 2g_t^2/l) \le \sqrt{b + 2g_t^2/l} + \sqrt{b}.
    \end{align*}
    Thus, $c$ just need to satisfy
    \begin{align*}
        c^2 \le \frac{b}{b + 2g_t^2/l}.
    \end{align*}
    According to $b \ge \beta_0^2$, $g_t^2 \le G^2$, hence
    \begin{align*}
        \frac{b}{b+2g_t^2/l} \ge \frac{\beta_0^2}{\beta_0^2 + 2G^2/l}.
    \end{align*}
    So if we choose $c = \sqrt{\frac{\beta_0^2}{\beta_0^2 + 2G^2/l}}$, then $\beta_1 > \beta_0 > c\beta_0$, hence
    \begin{align*}
        \beta_{t} \ge c \sqrt{\beta_0^2 + \frac{2}{l}\sum_{i=0}^{t-1}g_i^2}, t = 1, \cdots, T.
    \end{align*}
    Moreover, following from Lemma \ref{regret-basic-lemma}, we have
    \begin{align*}
        \sum_{t=0}^{T-1} \frac{g_t^2}{\beta_{t+1}}
        \le \sum_{t=0}^{T-1} \frac{g_t^2}{c\sqrt{2/l}\sqrt{\sum_{i=0}^{t}g_i^2}}
        \le \frac{\sqrt{2l}}{c}\sqrt{\sum_{t=0}^{T-1}g_t^2}.
    \end{align*}
\end{proof}

\begin{lemma}
    \label{algo-3-bound-lemma}
    Suppose an increasing function $\psi$ satisfies $\psi(1) = 0$ and $\psi(x) \le l(x - 1)$.
    Consider a real valued sequence $\{g_t\}_{t=0:T-1}$ and a positive sequence $\{\beta_t\}_{t=0:T}$ which satisfies
    $\vert g_t \vert \le G$, $\beta_{t}^2 \psi\left(\frac{\beta_{t+1}}{\beta_t}\right) = g_t^2$, $t=0,\cdots,T-1$, $\beta_0 \ge 0$.
    We can bound $\beta_T$ as
    \begin{align}
        \beta_{t} \ge \sqrt{\beta_0^2 + \frac{2}{l}\sum_{i=0}^{t-1}g_i^2}, t = 1, \cdots, T.
    \end{align}
    Moreover, we have
    \begin{align}
        \sum_{t=0}^{T-1} \frac{g_t^2}{\beta_{t}}
        \le \max\left\{\sqrt{2l}, \frac{2G}{\beta_0}\right\}
        \sqrt{\sum_{t=0}^{T-1}g_t^2}.
    \end{align}
\end{lemma}
\begin{proof}
    Same as inequality (\ref{algo-3-eq-1}), we have
    \begin{align*}
        l\beta_{t}^2\left(\frac{\beta_{t+1}}{\beta_t} - 1\right) \ge g_t^2,
    \end{align*}
    hence
    \begin{align*}
        \beta_{t+1}^2 = \left(\beta_t + \frac{g_t^2}{l\beta_t}\right)^2
        \ge \beta_t^2 + \frac{2}{l}g_t^2 \ge \beta_0^2 + \frac{2}{l}\sum_{i=0}^{t} g_t^2 \ge \min\left\{1, \frac{l\beta_0^2}{2G^2}\right\} \frac{2}{l}\sum_{i=0}^{t+1}g_i^2,.
    \end{align*}
    Furthermore, following from Lemma \ref{regret-basic-lemma},
    we have
    \begin{align*}
        \sum_{t=0}^{T-1} \frac{g_t^2}{\beta_{t}}
        \le \sqrt{\frac{l/2}{\min\{1, l\beta_0^2/(2G^2)\}}}
        \sum_{t=0}^{T-1} \frac{g_t^2}{\sqrt{\sum_{i=0}^{t}g_i^2}}
        \le \max\left\{\sqrt{2l}, \frac{2G}{\beta_0}\right\}
        \sqrt{\sum_{t=0}^{T-1}g_t^2}.
    \end{align*}
\end{proof}

\begin{thm}
    \label{spec-thm-algo-1}
    Suppose that $\varphi \in C_l^{1, 1} \left( [1, +\infty) \right)$, and $\varphi$ is $\gamma$-strongly convex.
    Assume that $\|\vg_t\|_{\infty} \le G$, and $\|\vx_t - \vx^*\|_{\infty} \le D_{\infty}$.
   Then the sequence $\{\vx_t\}$ obtained from Algorithm \ref{algo-1} satisfies
    \[
        2R(T) \le \left(1 + \frac{D_{\infty}^2}{\gamma}\right)\sqrt{2l + 4\alpha_0^2G^2} \sum_{j=1}^d \|g_{0:T-1,j}\|_2
        + \|\vx_0 - \vx^*\|_2^2 / \alpha_0.
    \]
\end{thm}
\begin{proof}
    Let $\vbeta_t = 1/\valpha_t$. Following from Lemma \ref{iteration-lemma},
    \begin{align*}
        2R(T) &\le \sum_{t=0}^{T-1} \|\vg_t\|_{\diag(\valpha_{t+1})}^2
        + \sum_{t=0}^{T-1} \|\vx_{t} - \vx^*\|_{\diag(\valpha_{t+1} - \valpha_{t})^{-1}}^2
        + \|\vx_0 - \vx^*\|_{\diag(\valpha_0)^{-1}}^2 \\
        &\le \sum_{t=0}^{T-1} \|\vg_t\|_{\diag(\vbeta_{t+1})^{-1}}^2
        + \sum_{t=0}^{T-1} \|\vx_{t} - \vx^*\|_{\infty}^2 \|\vbeta_{t+1} - \vbeta_{t}\|_1
        + \|\vx_0 - \vx^*\|_{\diag(\vbeta_0)}^2 \\
        &\le \sum_{t=0}^{T-1} \sum_{j=1}^d \frac{g_{t,j}^2}{\beta_{t+1,j}}
        + \max_{0 \le t < T}\|\vx_{t} - \vx^*\|_{\infty}^2 \sum_{t=0}^{T-1} \sum_{j=1}^d (\beta_{t+1,j} - \beta_{t,j})
        + \|\vx_0 - \vx^*\|_{\diag(\vbeta_0)}^2.
    \end{align*}
    Recall $\varphi$ is a $\gamma$-strongly convex function, and $\varphi'(\alpha_{t,j}/\alpha_{t+1,j}) = \alpha_{t+1,j}^2g_{t,j}^2$.\\
    so,
    \begin{align*}
        g_{t,j}^2 = \beta_{t+1,j}^2 \varphi'\left(\frac{\beta_{t+1,j}}{\beta_{t,j}}\right)
        \ge \gamma \beta_{t+1,j} \beta_{t,j} \left(\frac{\beta_{t+1,j}}{\beta_{t,j}} - 1\right),
    \end{align*}
    and
    \begin{align}
        \label{algo-3-proof-eq1}
        \sum_{t=0}^{T-1} (\beta_{t+1,j} - \beta_{t,j})
        \le \frac{1}{\gamma} \sum_{t=0}^{T-1} \frac{g_{t,j}^2}{\beta_{t+1,j}}.
    \end{align}
    The function $\psi = \varphi'$ satisfies $\psi(1) = 0$
    and $\psi(x) \le l(x - 1)$ according to the smoothness of $\varphi$.
    Following from Lemma \ref{algo-1-bound-lemma}, we have
    \begin{align}
        \label{algo-3-proof-eq2}
        \sum_{t=0}^{T-1} \frac{g_{t,j}^2}{\beta_{t+1,j}}
        \le \frac{\sqrt{2l\beta_{0,j}^2 + 4G^2}}{\beta_{0,j}} \sqrt{\sum_{i=0}^{T-1}g_{t,j}^2}
        =\frac{\sqrt{2l\beta_{0,j}^2 + 4G^2}}{\beta_{0,j}} \|g_{0:T-1,j}\|_2.
    \end{align}
    Combining inequality (\ref{algo-3-proof-eq1}) and (\ref{algo-3-proof-eq2}), we have
    \begin{align*}
        2R(T) &\le \left(1 + \frac{\max_{0 \le t < T}\|\vx_{t} - \vx^*\|_{\infty}^2}{\gamma}\right)
        \sum_{j=1}^d \sum_{t=0}^{T-1} \frac{g_{t,j}^2}{\beta_{t+1,j}}
        + \|\vx_0 - \vx^*\|_{\diag(\vbeta_0)}^2 \\
        &\le \left(1 + \frac{D_{\infty}^2}{\gamma}\right) \sum_{j=1}^d \frac{\sqrt{2l\beta_{0,j}^2 + 4G^2}}{\beta_{0,j}} \|g_{0:T-1,j}\|_2
        + \|\vx_0 - \vx^*\|_{\diag(\vbeta_0)}^2 \\
        &= \left(1 + \frac{D_{\infty}^2}{\gamma}\right)\frac{\sqrt{2l\beta_0^2 + 4G^2}}{\beta_0} \sum_{j=1}^d \|g_{0:T-1,j}\|_2
        + \beta_0 \|\vx_0 - \vx^*\|_2^2 \\
        &=\left(1 + \frac{D_{\infty}^2}{\gamma}\right)\sqrt{2l + 4\alpha_0^2G^2} \sum_{j=1}^d \|g_{0:T-1,j}\|_2
        + \|\vx_0 - \vx^*\|_2^2 / \alpha_0.
    \end{align*}
\end{proof}

\begin{thm}
    Suppose that $\varphi \in C_l^{1, 1} \left( [1, +\infty) \right)$, and $\varphi$ is $\alpha$-strongly convex.
    Assume that $\|\vg_t\|_{\infty} \le G$, and $\|\vx_t - \vx^*\|_{\infty} \le D_{\infty}$.
    Then the sequence $\{\vx_t\}$ obtained from Algorithm \ref{algo-2} satisfies
    \[
        2R(T) \le \left(1 + \frac{D_{\infty}^2}{\gamma}\right) 
        \max\left\{\sqrt{2l}, 2\alpha_0G\right\} \sum_{j=1}^d \|g_{0:T-1,j}\|_2
        + \|\vx_0 - \vx^*\|_2^2 / \alpha_0.
    \]
\end{thm}
\begin{proof}
    Let $\vbeta_t = 1/\valpha_t$. Similar to the proof of Theorem \ref{spec-thm-algo-1}, for Algorithm \ref{algo-2}, we have
    \begin{align*}
        2R(T) &\le \sum_{t=0}^{T-1} \sum_{j=1}^d \frac{g_{t,j}^2}{\beta_{t+1,j}}
        + \max_{0 \le t < T}\|\vx_{t} - \vx^*\|_{\infty}^2 \sum_{t=0}^{T-1} \sum_{j=1}^d (\beta_{t+1,j} - \beta_{t,j})
        + \|\vx_0 - \vx^*\|_{\diag(\vbeta_0)}^2 \\
        &\le \sum_{t=0}^{T-1} \sum_{j=1}^d \frac{g_{t,j}^2}{\beta_{t,j}}
        + \max_{0 \le t < T}\|\vx_{t} - \vx^*\|_{\infty}^2 \sum_{t=0}^{T-1} \sum_{j=1}^d (\beta_{t+1,j} - \beta_{t,j})
        + \|\vx_0 - \vx^*\|_{\diag(\vbeta_0)}^2.
    \end{align*}
    Note that in Algorithm \ref{algo-2}, $\alpha_{t,j}^2g_{t,j}^2 = \varphi'(\alpha_{t,j}/\alpha_{t+1,j})$, 
    thus
    \begin{align*}
        g_{t,j}^2 = \beta_{t,j}^2 \varphi'\left(\frac{\beta_{t+1,j}}{\beta_{t,j}}\right)
        \ge \gamma \beta_{t,j}^2 \left(\frac{\beta_{t+1,j}}{\beta_{t,j}} - 1\right),
    \end{align*}
    and
    \begin{align*}
        \sum_{t=1}^{T-1} (\beta_{t,j} - \beta_{t-1,j})
        \le \sum_{t=0}^{T-1} (\beta_{t+1,j} - \beta_{t,j})
        \le \frac{1}{\gamma} \sum_{t=0}^{T-1} \frac{g_{t,j}^2}{\beta_{t,j}}.
    \end{align*}
    Thus, following from Lemma \ref{algo-3-bound-lemma} and similar reason in our proof of Theorem \ref{spec-thm-algo-1}, we have
    \begin{align*}
        2R(T) &\le \left(1 + \frac{D_{\infty}^2}{\gamma}\right) \sum_{j=1}^d \sum_{t=0}^{T-1} \frac{g_{t,j}^2}{\beta_{t,j}}
        + \beta_0 \|\vx_0 - \vx^*\|_2^2 \\
        &\le \left(1 + \frac{D_{\infty}^2}{\gamma}\right) 
        \max\left\{\sqrt{2l}, \frac{2G}{\beta_0}\right\} \sum_{j=1}^d \|g_{0:T-1,j}\|_2
        + \beta_0 \|\vx_0 - \vx^*\|_2^2 \\
        &= \left(1 + \frac{D_{\infty}^2}{\gamma}\right) 
        \max\left\{\sqrt{2l}, 2\alpha_0G\right\} \sum_{j=1}^d \|g_{0:T-1,j}\|_2
        + \|\vx_0 - \vx^*\|_2^2 / \alpha_0.
\end{align*}
\end{proof}

\section{Logarithmic Bounds}
\label{sec-log-appendix}
In this section, we will use a different class of `distance' function for problem (\ref{saddle-point-problem}), 
and establish logarithmic regret bounds under assumption $f_t$ is strongly convex.
Our analysis and proof follow from \cite{hazan2007logarithmic, mukkamala2017scadagrad}.

First, we define $\vmu$-strongly convexity.
\begin{defn}[Definition 2.1 in \cite{mukkamala2017scadagrad}]
    Let $\gX \subseteq \sR^d$ be a convex set. 
    We say that a function $f : \gX \rightarrow \sR$ is $\vmu$-strongly convex, 
    if there exists $\vmu \in \sR^d$ with $\mu_j > 0$ for $j = 1, \cdots, d$ 
    such that for all $\vx, \vy \in \gX$,
    \begin{align*}
        f(\vy) \ge f(\vx) + \langle \nabla f(\vx), \vy -\vx \rangle + \frac{1}{2}\|\vy - \vx\|_{\diag(\vmu)}^2.
    \end{align*}
    Let $\xi = \min_{j=1:d}\mu_j$, then this function is $\xi$-strongly convex (in the usual sense),
    that is 
    \begin{align*}
        f(\vy) \ge f(\vx) + \langle \nabla f(\vx), \vy -\vx \rangle + \frac{\xi}{2}\|\vy - \vx\|_2^2.
    \end{align*}
\end{defn}

The modification SC-Meta-Regularization of Meta-Regularization which we propose
in the following uses a family of distance function $D: \sR_{++}^d \times \sR_{++}^d \rightarrow \sR$ formulated as
\begin{align}
    \label{new-distance}
    D(\vu, \vv) = \sum_{j=1}^d \varphi(v_j/u_j),
\end{align}
where $\varphi$ is convex function with $\varphi(1) = \varphi'(1) = 0$ like we used in $\varphi$-divergence.

\begin{remark}
    Same as $\varphi$-divergence, $D(\vu, \vv) \ge 0$ ~for any $\vu, \vv \in \sR_{++}^d$.
\end{remark}

Different from Algorithm \ref{algo-1} and \ref{algo-2}, 
we add a hyper-parameter $\lambda > 0$ like AdaGrad to SC-AdaGrad.
Rewrite problem (\ref{saddle-point-problem}) as
\begin{align}
    \max_{\valpha \in \mathcal{A}_t} \min_{\vx \in \mathcal{X}} \Psi_t(\vx, \valpha) \triangleq
    \vg_t^{\top} (\vx - \vx_t) + \frac{1}{2}\|\vx - \vx_t\|_{\diag(\valpha)^{-1}}^2 -
    \frac{\lambda}{2} \sum_{j=1}^d \varphi(\alpha_{t,j}/\alpha_j).
\end{align}
Similarly, we can also derive two algorithms according to two update rules respectively.
\begin{algorithm}[H]
\caption{GD with SC-Meta-Regularization (Algorithm \ref{algo-log} in Section \ref{sec-log})} \label{algo-4}
\begin{algorithmic}[1]
\REQUIRE $\valpha_0 > 0$, $\vx_0$
\FOR{$t = 1$ to $T$}
\STATE Suffer loss $f_t(\vx_t)$;
\STATE Receive $\vg_t \in \partial f_t(\vx_t)$ of $f_t$ at $\vx_t$;
\STATE Update $\alpha_{t+1, j}$ as the solution of the equation $\lambda (\alpha_{t,j}/\alpha^2 ) \varphi' (\alpha_{t,j}/\alpha) = g_{t,j}^2, j = 1, \cdots, d$;
\STATE Update $\vx_{t+1} = \vx_t - \valpha_{t+1}\vg_t$;
\ENDFOR
\end{algorithmic}
\end{algorithm}
\begin{algorithm}[H]
\caption{GD with SC-Meta-Regularization using alternating update rule} \label{algo-5}
\begin{algorithmic}[1]
\REQUIRE $\valpha_0 > 0$, $\vx_0$
\FOR{$t = 1$ to $T$}
\STATE Suffer loss $f_t(\vx_t)$;
\STATE Receive $\vg_t \in \partial f_t(\vx_t)$ of $f_t$ at $\vx_t$;
\STATE Update $\alpha_{t+1, j} = \alpha_{t, j}/ (\varphi')^{-1}(\alpha_{t, j}g_{t, j}^2 / \lambda),$ $j = 1, \ldots, d$;
\STATE Update $\vx_{t+1} = \vx_t - \valpha_{t+1}\vg_t$;
\ENDFOR
\end{algorithmic}
\end{algorithm}
\begin{remark}
    Same as Lemma \ref{lemma-monotonicity}, the monotonicity of Algorithm \ref{algo-4} and \ref{algo-5} also holds.
\end{remark}

\begin{thm}
    \label{algo-45-thm}
    Suppose that $f_t$ is $\vmu$-strongly convex for all $t$, $\varphi \in C_l^{1, 1} \left( [1, +\infty) \right)$, and $\varphi$ is $\gamma$-strongly convex.
    Assume that $\|\vg_t\|_{\infty} \le G$, and $\lambda \ge G^2 / (\gamma\min_{j=1:d}\mu_j)$.
    Then the sequence $\{\vx_t\}$ obtained from Algorithm \ref{algo-4} satisfies
    \[
        2R(T) \le l \left(1 + \frac{\alpha_0G^2}{\lambda l}\right)^2 \sum_{j=1}^d 
        \ln \left( 1 + \frac{\alpha_0\norm{g_{0:T-1,j}}}{\lambda l} \right)
        + \|\vx_0 - \vx^*\|_2^2/\alpha_0,
    \]
    and the sequence $\{\vx_t\}$ obtained from Algorithm \ref{algo-5} satisfies
    \[
        2R(T) \le l \sum_{j=1}^d \ln \left( 1 + \frac{\alpha_0\norm{g_{0:T-1,j}}}{\lambda l } \right)
        + \|\vx_0 - \vx^*\|_2^2/\alpha_0.
    \]
\end{thm}
\begin{remark}
    Under assumption in Theorem \ref{algo-45-thm}, we have $\norm{g_{0:T-1,j}} \le G^2 T$, so  $R(T) = \gO(\ln(T))$.
\end{remark}
To prove Theorem \ref{algo-45-thm}, we first prove following lemma.
\begin{lemma}
    \label{log-regret-basic-lemma}
    For an arbitrary real-valued sequence \{$a_i$\} and a positive real number $b$,
    \begin{align}
        \sum_{t=1}^{T} \frac{a_t^2}{b + \sum_{i=1}^t a_i^2} \le \ln\left( 1 + \frac{\sum_{t=1}^T a_t^2}{b} \right).
    \end{align}
\end{lemma}
\begin{proof}
    Let $b_0 = b, b_t = b + \sum_{i=1}^t a_i^2, t \ge 1$, then
    \begin{align*}
        &\quad \sum_{t=1}^{T} \frac{a_t^2}{b + \sum_{i=1}^t a_i^2} 
        = \sum_{t=1}^{T} \frac{b_t - b_{t-1}}{b_t}
        = \sum_{t=1}^{T} \int_{b_{t-1}}^{b_t} \frac{1}{b_t} d x \\
        &\le \sum_{t=1}^{T} \int_{b_{t-1}}^{b_t} \frac{1}{x} d x 
        = \int_{b}^{b_T} \frac{1}{x} d x 
        = \ln\left( 1 + \frac{\sum_{t=1}^T a_t^2}{b} \right).
    \end{align*}
\end{proof}
Like Lemma \ref{algo-1-bound-lemma} and \ref{algo-3-bound-lemma}, similar lemma holds for Algorithm \ref{algo-4} and \ref{algo-5}.
\begin{lemma}
    \label{algo-45-lemma}
    Suppose an increasing function $\psi$ satisfies $\psi(1) = 0$ and $\psi(x) \le l(x - 1)$.
    Consider a real valued sequence $\{g_t\}_{t=0:T-1}$ and a positive sequence $\{\beta_t\}_{t=0:T}$ which satisfies
    $\vert g_t \vert \le G$, $\beta_0 > 0$. \\
    If $(\beta_{t+1}^2 / \beta_t) \psi(\beta_{t+1}/\beta_t) = g_t^2$, $t=0,\cdots,T-1$, 
    then we have
    \begin{align}
        \beta_{t} \ge \left(\frac{\beta_0}{\beta_0 + G^2/l}\right)^2 \left(\beta_0 + \frac{1}{l}\sum_{i=0}^{t-1}g_i^2\right), \ t = 1, \cdots, T
    \end{align}
    and 
    \begin{align}
        \sum_{t=0}^{T-1} \frac{g_t^2}{\beta_{t+1}}
        \le l \left(\frac{\beta_0 + G^2/l}{\beta_0}\right)^2 \ln \left( 1 + \frac{\sum_{t=0}^{T-1}g_t^2}{l\beta_0} \right).
    \end{align}
    Meanwhile, if $\beta_{t} \psi(\beta_{t+1}/\beta_t) = g_t^2$, $t=0,\cdots,T-1$, 
    then we have
    \begin{align}
        \beta_{t} \ge \beta_0 + \frac{1}{l}\sum_{i=0}^{t-1}g_i^2, \ t = 1, \cdots, T
    \end{align}
    and 
    \begin{align}
        \sum_{t=0}^{T-1} \frac{g_t^2}{\beta_{t+1}}
        \le l \ln \left( 1 + \frac{\sum_{t=0}^{T-1}g_t^2}{l\beta_0} \right).
    \end{align}
\end{lemma}
\begin{proof}
    Using same methods in proof of Lemma \ref{algo-1-bound-lemma} and \ref{algo-3-bound-lemma}, the conclusion can be deduced from Lemma \ref{log-regret-basic-lemma} easily.
\end{proof}

\begin{proof}[\textbf{proof of Theorem \ref{algo-45-thm}}]
    Like Lemma \ref{iteration-lemma}, in strongly convex case, we have
    \begin{align*}
        2R(T) &= 2\sum_{t=0}^{T-1} f_t(\vx_t) - f_t(\vx^*) \\
        &\le 2\sum_{t=0}^{T-1} \langle \vg_t, \vx_t - \vx^* \rangle 
        - \sum_{t=0}^{T-1} \|\vx_t - \vx^*\|_{\diag(\vmu)}^2 \\
        &= \sum_{t=0}^{T-1} \|\vg_t\|_{\diag(\valpha_{t+1})}^2
        + \sum_{t=0}^{T-1} \left(\|\vx_{t} - \vx^*\|_{\diag(\valpha_{t+1})^{-1}}^2
        - \|\vx_{t+1} - \vx^*\|_{\diag(\valpha_{t+1})^{-1}}^2\right) 
        - \sum_{t=0}^{T-1} \|\vx_t - \vx^*\|_{\diag(\vmu)}^2 \\
        &\le \sum_{t=0}^{T-1} \|\vg_t\|_{\diag(\valpha_{t+1})}^2
        + \sum_{t=0}^{T-1} \|\vx_{t} - \vx^*\|_{\diag(1/\valpha_{t+1} - 1/\valpha_{t} - \vmu)}^2
        + \|\vx_0 - \vx^*\|_{\diag(\valpha_0)^{-1}}^2.
    \end{align*}
    Note that in Algorithm \ref{algo-4}, $\lambda (\alpha_{t,j}/\alpha_{t+1,j}^2 ) \varphi' (\alpha_{t,j}/\alpha_{t+1,j}) = g_{t,j}^2$, so
    \begin{align*}
        &\quad \frac{1}{\alpha_{t+1,j}} - \frac{1}{\alpha_{t,j}} = \frac{1}{\alpha_{t,j}} \left(\frac{\alpha_{t,j}}{\alpha_{t+1,j}} - 1\right) \\
        &\le \frac{1}{\gamma\alpha_{t,j}} \varphi'\left(\frac{\alpha_{t,j}}{\alpha_{t+1,j}}\right)
        = \frac{\alpha_{t+1,j}^2}{\alpha_{t,j}^2} \frac{g_{t,j}^2}{\lambda\gamma} \le \frac{G^2}{\lambda\gamma}.
    \end{align*}
    And in Algorithm \ref{algo-5}, $\alpha_{t+1, j} = \alpha_{t, j}/ (\varphi')^{-1}(\alpha_{t, j}g_{t, j}^2 / \lambda)$, thus same conclusion holds:
    \begin{align*}
        &\quad \frac{1}{\alpha_{t+1,j}} - \frac{1}{\alpha_{t,j}} = \frac{1}{\alpha_{t,j}} \left(\frac{\alpha_{t,j}}{\alpha_{t+1,j}} - 1\right) \\
        &\le \frac{1}{\alpha_{t,j}\gamma} \varphi'\left(\frac{\alpha_{t,j}}{\alpha_{t+1,j}}\right)
        = \frac{g_{t,j}^2}{\lambda\gamma} \le \frac{G^2}{\lambda\gamma}.
    \end{align*}
    Hence, if $\lambda \ge \max_{j=1:d}\frac{G^2}{\gamma\mu_j}$, then ~$1/\valpha_{t+1} - 1/\valpha_{t} \le \vmu$, and
    \begin{align*}
        \sum_{t=0}^{T-1} \|\vx_{t} - \vx^*\|_{\diag(1/\valpha_{t+1} - 1/\valpha_{t} - \vmu)}^2 \le 0.
    \end{align*}
    On the other hand, let $\vbeta_t = 1 / \valpha_t$,
    \begin{align*}
        \sum_{t=0}^{T-1} \|\vg_t\|_{\diag(\valpha_{t+1})}^2
        =\sum_{j=1}^d \sum_{t=0}^{T-1} \frac{g_{t,j}^2}{\beta_{t+1,j}},
    \end{align*}
    following from Lemma \ref{algo-45-lemma}, we have
    \begin{align*}
    \sum_{t=0}^{T-1} \|\vg_t\|_{\diag(\valpha_{t+1})}^2 
    &\le l \left(1 + \frac{G^2}{\lambda l \beta_0}\right)^2 \sum_{j=1}^d \ln \left( 1 + \frac{\norm{g_{0:T-1,j}}}{\lambda l \beta_0} \right) \text{  in Algorithm \ref{algo-4}, } \\
    \sum_{t=0}^{T-1} \|\vg_t\|_{\diag(\valpha_{t+1})}^2
    &\le l \sum_{j=1}^d \ln \left( 1 + \frac{\norm{g_{0:T-1,j}}}{\lambda l \beta_0} \right) \text{  in Algorithm \ref{algo-5}. }
    \end{align*}
\end{proof}

\section{Run-time in Full batch Setting}
\label{sec-full-batch}
In this section, we will discuss the convergence of our methods in full batch setting.

We first review a classical result on the convergence rate for gradient descent with fixed learning rate.

\begin{thm}
\label{classical-result}
Suppose that $F \in C_L^{1, 1}(\sR^d)$ and $F^* = inf_{\vx} F(\vx) > -\infty $. Consider gradient descent
with constant step size, $\vx_{t+1} = \vx_t - \frac{\nabla F(\vx_t)}{b}$. If $b > \frac{L}{2}$,
then
\begin{align*}
\min_{0 \le t \le T-1} \norm{\nabla F(\vx_t)} \le \varepsilon
\end{align*}
after at most a number of steps
\begin{align*}
T = \frac{2b^2 (F(\vx_0) - F^*)}{\varepsilon (2b - L)} = \mathcal{O} \left(\frac {1}{\varepsilon} \right)
\end{align*}
\end{thm}
\begin{proof}
Following from the fact that $F$ is $L$-smooth, we have
\begin{align}
\notag F(\vx_{t+1}) \le& F(\vx_t) + \nabla F(\vx_t)^{\top} (\vx_{t+1} - \vx_t) + \frac{L}{2} \norm{\vx_{t+1} - \vx_t} \\
\notag =& F(\vx_t) - \frac{1}{b} \norm{\nabla F(\vx_t)} + \frac{L}{2b^2} \norm{\nabla F(\vx_t)} \\
\label{smooth-F} =& F(\vx_t) - \frac{1}{b}\left(1 - \frac{L}{2b} \right) \norm{\nabla F(\vx_t)}.
\end{align}
When $b > \frac{L}{2}$, $1 - \frac{L}{2b} > 0$. So
\begin{align*}
\sum_{t=0}^{T-1} \norm{\nabla F(\vx_t)} \le \frac{2b^2}{2b - L}(F(\vx_0) - F(\vx_T)) \le \frac{2b^2}{2b - L}(F(\vx_0) - F^*),
\end{align*}
and
\begin{align*}
\min_{0 \le t \le T-1} \norm{\nabla F(\vx_t)} \le \frac{1}{T} \sum_{t=0}^{T-1} \norm{\nabla F(\vx_t)}
\le \frac{2b^2}{T(2b - L)}(F(\vx_0) - F^*) \le \varepsilon.
\end{align*}
\end{proof}

\begin{remark}
If we choose $b \le \frac{L}{2}$, then convergence of gradient descent with constant learning rate is not guaranteed at all.
\end{remark}


Next we will show that convergence of both update rules (\ref{scalar-1}) and (\ref{scalar-2}) are robust to the choice of initial learning rate.
Our proof is followed from the proof of Theorem 2.3 in WNGrad \cite{wu2018wngrad}.

We denote the reciprocal of learning rate $\alpha_t$ by $\beta_t$, i.e., $\beta_t = 1 / \alpha_t$.
Note that in update rule (\ref{scalar-1}), $\beta_{t+1}$ satisfies 
\begin{align*}
    \beta_{t+1}^2 \varphi'(\beta_{t+1}/\beta_t) = \norm{\vg_t},
\end{align*}
while in update rule (\ref{scalar-2}), $\beta_{t+1}$ satisfies
\begin{align*}
    \beta_t^2 \varphi'(\beta_{t+1}/\beta_t) = \norm{\vg_t}.
\end{align*}

Following Theorem \ref{scalar-1-thm} and \ref{scalar-2-thm} are detailed version of Theorem \ref{fullbatch-thm}.

\begin{thm}[Run-time of update rule (\ref{scalar-1})]
    \label{scalar-1-thm}
    Suppose that $\varphi \in C_l^{1, 1} \left( [1, +\infty) \right)$, $\varphi$ is $\gamma$-strongly convex, 
    and $F \in C_L^{1,1}(\sR^d)$, $F^* = \inf_{\vx}F(\vx) > -\infty$.
    For any $\varepsilon \in (0, 1)$, the sequence $\{\vx_t\}$ obtained from update rule (\ref{scalar-1}) satisfies
    \begin{align*}
        \min_{j=0:T-1}\norm{\nabla F(\vx_j)} \le \varepsilon,
    \end{align*}
    after $T$ steps, where
    \begin{align*}
    T =
    \begin{cases}
        1 + \left\lceil \frac{2(\beta_0 + 2(F(\vx_0)-F^*)/\gamma)(F(\vx_0)-F^*)}{\varepsilon} \right\rceil
        \text{ if } \beta_0 \ge L \text{ or } \beta_1 \ge L, \\
        1 + \left\lceil \frac{\log(\frac{L}{\beta_0})}{\log(\frac{\varepsilon}{lL^2} + 1)} \right\rceil +
        \left\lceil \frac{\left(L + \left(1 + \frac{2}{\gamma}\right)\left(F(\vx_0) - F^* + \frac{lL(L-\beta_0)}{2\beta_0}\right)\right)^2}{\varepsilon} \right\rceil
        \text{ otherwise.}
    \end{cases}
    \end{align*}
\end{thm}

\begin{thm}[Run-time of update rule (\ref{scalar-2})]
    \label{scalar-2-thm}
    Suppose that $\varphi \in C_l^{1, 1} \left( [1, +\infty) \right)$, $\varphi$ is $\gamma$-strongly convex, 
    and $F \in C_L^{1,1}(\sR^d)$, $F^* = \inf_{\vx}F(\vx) > -\infty$.
    For any $\varepsilon \in (0, 1)$, the sequence $\{\vx_t\}$ obtained from update rule (\ref{scalar-2}) satisfies
    \begin{align*}
        \min_{j=0:T-1}\norm{\nabla F(\vx_j)} \le \varepsilon
    \end{align*}
    after $T$ steps, where 
    \begin{align*}
    T =
    \begin{cases}
        1 + \left\lceil \frac{2(\beta_0 + \norm{\vg_0}/(\gamma\beta_0) + 2(F(\vx_0)-F^*)/\gamma)(F(\vx_0)-F^*)}{\varepsilon} \right\rceil
        \text{ if } \beta_0 \ge L \text{ or } \beta_1 \ge L, \\
1 + \left\lceil \frac{\log(\frac{L}{\beta_0})}{\log(\frac{\varepsilon}{lL^2} + 1)} \right\rceil +
        \left\lceil \frac{\left(L + \frac{2l}{\gamma\beta_0}L^2 + \frac{2l}{\gamma}L + \left(1 + \frac{8}{\gamma}\right)\left(F(\vx_0) - F^* + \frac{lL(L-\beta_0)}{2\beta_0}\right)\right)^2}{\varepsilon} \right\rceil 
        \text{ otherwise.}
    \end{cases}
    \end{align*}
\end{thm}


We begin our proof by following lemma.
\begin{lemma}
    \label{lemma-always-ge-L}
    Suppose $\varphi \in C_l^{1,1}(\sR_{++})$.
    Fix $\varepsilon \in (0, 1]$. In both update rules (\ref{scalar-1}) and (\ref{scalar-2}),
    after $T = \left\lceil \frac{\log(\frac{L}{\beta_0})}{\log(\frac{\varepsilon}{lL^2} + 1)} \right\rceil + 1$
    steps, either $\min_{t=0:T-1} \norm{\vg_t} \le \varepsilon$, or $\beta_T \ge L$ holds.
\end{lemma}
\begin{proof}
    Assume that $\beta_T < L$ and $\min_{t=0:T-1} \norm{\vg_t} > \varepsilon$.
    Recall that the sequence \{$\beta_t$\} is an increasing sequence.
    Hence, $\beta_t < L$ for $0 \le t \le T$.\\
    So, for all $0 \le t \le T - 1$,
    \begin{align*}
        \varphi'\left(\frac{\beta_{t+1}}{\beta_t}\right) =
        \frac{\norm{\vg_t}}{\beta_{t+1}^2} > \frac{\varepsilon}{L^2}
        \text{    (for update rule (\ref{scalar-1}))},\\
        \varphi'\left(\frac{\beta_{t+1}}{\beta_t}\right) =
        \frac{\norm{\vg_t}}{\beta_{t}^2} > \frac{\varepsilon}{L^2}
        \text{    (for update rule (\ref{scalar-2}))}.
    \end{align*}
    Note that $\varphi$ is a $l$-smooth convex function, and $\beta_{t+1}/\beta_t \ge 1$.
    So
    \begin{align}
        \label{smooth-phi}
        \varphi'\left(\frac{\beta_{t+1}}{\beta_t}\right) \le
        l\left(\frac{\beta_{t+1}}{\beta_t} - 1\right),
    \end{align}
    then
    \begin{align*}
        \frac{\beta_{t+1}}{\beta_t} > \frac{\varepsilon}{lL^2} + 1.
    \end{align*}
    In this case,
    \begin{align*}
        L > \beta_T = \beta_0 \left(\frac{\varepsilon}{lL^2} + 1\right)^{T},
    \end{align*}
    however, it is impossible according to the setting of $T$ in the lemma.
\end{proof}

We first prove Theorem \ref{scalar-1-thm} using following lemma.
\begin{lemma}
    \label{scalar-1-lemma}
    In update rule (\ref{scalar-1}), suppose $F \in C_L^{1,1}(\sR^d)$, $\varphi \in C_l^{1,1}(\sR_{++})$,
    and $\varphi$ is $\gamma$-strongly convex function.
    Denote $F^* = \inf_{\vx}F(\vx)>-\infty$.
    Let $t_0 \ge 1$ be the first index such that $\beta_{t_0} \ge L$.
    Then for all $t \ge t_0$,
    \begin{align}
        \beta_t \le \beta_{t_0-1} + \frac{2}{\gamma}(F(\vx_{t_0-1}) - F^*),
    \end{align}
    and moreover,
    \begin{align}
        \label{Fxt0-1}
        F(\vx_{t_0-1}) - F^* \le F(\vx_0) - F^* + \frac{Ll}{2\beta_0}(\beta_{t_0-1} - \beta_0)
    \end{align}
\end{lemma}
\begin{proof}
    Same as equation (\ref{smooth-F}),
    \begin{align*}
        F(\vx_{t+1}) \le F(\vx_t) - \frac{1}{\beta_{t+1}}\left(1 - \frac{L}{2\beta_{t+1}}\right) \norm{\vg_t}.
    \end{align*}
    For $t \ge t_0-1$, $\beta_{t+1} \ge L$, so
    \begin{align*}
        F(\vx_{t+1}) \le F(\vx_t) - \frac{1}{2\beta_{t+1}} \norm{\vg_t}.
    \end{align*}
    Hence, for all $k \ge 0$,
    \begin{align}
        \label{scalar-1-lemma-eq-3}
        F(\vx_{t_0+k}) \le F(\vx_{t_0-1}) - \frac{1}{2}\sum_{i=0}^k\frac{\norm{\vg_{t_0+i-1}}}{\beta_{t_0+i}},
    \end{align}
    i.e.,
    \begin{align}
        \label{scalar-1-lemma-eq-1}
        \sum_{i=0}^k\frac{\norm{\vg_{t_0+i-1}}}{\beta_{t_0+i}} \le 2(F(\vx_{t_0-1}) - F^*).
    \end{align}
    Note that $\varphi$ is $\gamma$-strongly convex
    and $\beta_{t+1}^2 \varphi'(\beta_{t+1}/\beta_t) = \norm{\vg_t}$.
    So
    \begin{align*}
        \frac{\norm{\vg_t}}{\beta_{t+1}} =
        \beta_{t+1}\varphi'\left(\frac{\beta_{t+1}}{\beta_t}\right)
        \ge \gamma \beta_{t} \left(\frac{\beta_{t+1}}{\beta_t} - 1\right),
    \end{align*}
    and
    \begin{align}
        \label{scalar-1-lemma-eq-2}
        \beta_{t+1} - \beta_{t} \le \frac{1}{\gamma} \frac{\norm{\vg_t}}{\beta_{t+1}}.
    \end{align}
    Combining equation (\ref{scalar-1-lemma-eq-1}) and equation (\ref{scalar-1-lemma-eq-2}),
    we have
    \begin{align*}
        \beta_{t_0+k} &\le \beta_{t_0-1} + \frac{1}{\gamma} \sum_{i=0}^k\frac{\norm{\vg_{t_0+i-1}}}{\beta_{t_0+i}} \\
        &\le \beta_{t_0-1} + \frac{2}{\gamma} (F(\vx_{t_0-1}) - F^*).
    \end{align*}
    We remain to give an a upper bound for $F(\vx_{t_0-1})$ in the case $t_0 > 1$.
    Using equation (\ref{smooth-F}) again, we get
    \begin{align*}
        F(\vx_{t_0-1}) - F(\vx_0) &\le
        \sum_{i=0}^{t_0-2} - \frac{1}{\beta_{i+1}}\left(1 - \frac{L}{2\beta_{i+1}} \right) \norm{\vg_i} 
        \le \frac{L}{2} \sum_{i=0}^{t_0-2} \frac{\norm{\vg_i}}{\beta_{i+1}^2} \\
        &= \frac{L}{2} \sum_{i=0}^{t_0-2} \varphi' \left(\frac{\beta_{i+1}}{\beta_i}\right)
        \le \frac{Ll}{2} \sum_{i=0}^{t_0-2} \left(\frac{\beta_{i+1}}{\beta_i} - 1\right) \\
        &\le \frac{Ll}{2} \sum_{i=0}^{t_0-2} \left(\frac{\beta_{i+1} - \beta_i}{\beta_0}\right) 
        = \frac{Ll}{2\beta_0} (\beta_{t_0-1} - \beta_0).
    \end{align*}
    In the above, the second inequality follows from the assumed $l$-smoothness of $\varphi$,
    and the last inequality follows from $\beta_t \ge \beta_0$ for all $t \ge 0$.
\end{proof}

\begin{proof}[\textbf{proof of Theorem \ref{scalar-1-thm}}]
    If $t_0 = 1$, by equation (\ref{scalar-1-lemma-eq-3}), for all $t \ge 1$, we have
    \begin{align*}
        F(\vx_{t}) &\le F(\vx_0) - \frac{1}{2}\sum_{i=0}^{t-1} \frac{\norm{\vg_i}}{\beta_{i+1}} \\
        &\le F(\vx_0) - \frac{1}{2}\sum_{i=0}^{t-1}\frac{\norm{\vg_i}}{\beta_0 + \frac{2}{\gamma}(F(\vx_0) - F^*)}.
    \end{align*}
    Then after $T=1 + \left\lceil \frac{2(\beta_0 + 2(F(\vx_0)-F^*)/\gamma)(F(\vx_0)-F^*)}{\varepsilon} \right\rceil$
    steps,
    \begin{align*}
        \min_{t=0:T-1}\norm{\vg_t} &\le \frac{1}{T}\sum_{t=0}^{T-1} \norm{\vg_t} \\
        &\le \frac{2}{T}(F(\vx_0) - F^*)(\beta_0 + \frac{2}{\gamma}(F(\vx_0) - F^*)) \le \varepsilon.
    \end{align*}
    Otherwise, if $t_0 > 1$, we have $\beta_{t_0-1} < L$.
    Then for all $t \ge t_0$,
    \begin{align}
        \label{beta-upper-bound}
        \beta_t \le L + \frac{2}{\gamma}\left(F(\vx_0) - F^* + \frac{lL(L-\beta_0)}{2\beta_0}\right)
    \end{align}
    Denote the right hand of equation (\ref{beta-upper-bound}) as $\beta_{max}$.
    Using equation (\ref{scalar-1-lemma-eq-3}) again, for we have
    \begin{align*}
        F(\vx_{t_0+M}) &\le F(\vx_{t_0-1}) - \frac{1}{2}\sum_{i=0}^{M} \frac{\norm{\vg_{t_0+i-1}}}{\beta_{t_0+i}} \\
        &\le F(\vx_{t_0-1}) - \frac{1}{2\beta_{max}} \sum_{i=0}^{M} \norm{\vg_{t_0+i-1}}.
    \end{align*}
    Hence,
    \begin{align*}
        \min_{t=0:t_0+M-1} \norm{\vg_t} &\le \min_{t=t_0-1:t_0+M-1} \norm{\vg_t} \\
        &\le \frac{1}{M+1} \sum_{i=0}^M \norm{\vg_{t_0+i-1}} \\
        &\le \frac{1}{M+1} 2\beta_{max}(F(\vx_{t_0-1}) - F^*) \\
        &\le \frac{2\beta_{max}}{M+1} \left(F(\vx_0) - F^* + \frac{l L(L-\beta_0)}{2\beta_0}\right).
    \end{align*}
    At last, with recalling the conclusion of Lemma \ref{lemma-always-ge-L},
    after
    \begin{align*}
        T = \left\lceil \frac{\log(\frac{L}{\beta_0})}{\log(\frac{\varepsilon}{l L^2} + 1)} \right\rceil +
        \left\lceil \frac{2\beta_{max}}{\varepsilon} \left(F(\vx_0) - F^* + \frac{l L(L-\beta_0)}{2\beta_0}\right) \right\rceil + 1
    \end{align*}
    steps, we have $\min_{t=0:T-1}\norm{\vg_t} \le \varepsilon$.
\end{proof}

Next we prove Theorem \ref{scalar-2-thm}. 
\begin{lemma}
    \label{scalar-2-lemma}
    In update rule (\ref{scalar-2}), suppose $F \in C_L^{1,1}(\sR^d)$, $\varphi \in C_l^{1,1}(\sR_{++})$,
    and $\varphi$ is $\gamma$-strongly convex function.
    Denote $F^* = \inf_{\vx}F(\vx)$.
    Let $t_0 \ge 1$ be the first index such that $\beta_{t_0} \ge L$.
    Then for all $t \ge t_0$,
    \begin{align}
        \beta_t \le \beta_{t_0} + \frac{8}{\gamma}(F(\vx_{t_0-1}) - F^*),
    \end{align}
    and moreover,
    \begin{align}
        F(\vx_{t_0-1}) - F^* \le F(\vx_0) - F^* + \frac{Ll}{2\beta_0}(\beta_{t_0-1} - \beta_0), \\
        \beta_{t_0} \le
        \begin{cases}
            \beta_0 + \frac{\norm{\vg_0}}{\gamma\beta_0} \text{  if } t_0 = 1, \\
            L + \frac{2l}{\gamma\beta_0}L^2 + \frac{2l}{\gamma}L \text{  if } t_0 \ge 2,
        \end{cases}
    \end{align}
\end{lemma}
\begin{proof}
    Same as the proof of Lemma \ref{scalar-1-lemma}, we first get
    for all $k \ge 0$,
    \begin{align*}
        \sum_{i=0}^{k} \frac{\norm{\vg_{t_0+i-1}}}{\beta_{t_0+i}} \le 2 (F(\vx_{t_0-1}) - F^*).
    \end{align*}
    Note that in update rule (\ref{scalar-2}), $\beta_t^2 \varphi' \left(\beta_{t+1}/\beta_{t}\right) = \norm{\vg_t}$.
    So
    \begin{align*}
        \beta_{t_0+k+1} &= \beta_{t_0+k} +
        \beta_{t_0+k} \left(\frac{\beta_{t_0+k+1}}{\beta_{t_0+k}} - 1\right) \\
        &\le \beta_{t_0+k} +
        \frac{\beta_{t_0+k}}{\gamma} \varphi' \left(\frac{\beta_{t_0+k+1}}{\beta_{t_0+k}}\right) 
        = \beta_{t_0+k} + \frac{1}{\gamma}\frac{\norm{\vg_{t_0+k}}}{\beta_{t_0+k}} \\
        &\le \beta_{t_0+k} + \frac{2}{\gamma} \frac{\norm{\vg_{t_0+k}-\vg_{t_0+k-1}} 
        + \norm{\vg_{t_0+k-1}}}{\beta_{t_0+k}} \\
        &\le \beta_{t_0+k} + \frac{2}{\gamma} \frac{L^2\norm{\vx_{t_0+k}-\vx_{t_0+k-1}} 
        + \norm{\vg_{t_0+k-1}}}{\beta_{t_0+k}} \\
        &\le \beta_{t_0+k} + \frac{2}{\gamma} \frac{L^2\norm{\vg_{t_0+k-1}}}{\beta_{t_0+k}^3}
        + \frac{2}{\gamma} \frac{\norm{\vg_{t_0+k-1}}}{\beta_{t_0+k}} \\
        &\le \beta_{t_0+k} + \frac{4}{\gamma} \frac{\norm{\vg_{t_0+k-1}}}{\beta_{t_0+k}} 
        \le \beta_{t_0} + \frac{4}{\gamma} \sum_{i=0}^{k} \frac{\norm{\vg_{t_0+i-1}}}{\beta_{t_0+i}} \\
        &\le \beta_{t_0} + \frac{8}{\gamma}(F(\vx_{t_0-1}) - F^*).
    \end{align*}
    If $t_0 = 1$, then 
    \begin{align*}
        \beta_{t_0} &\le \beta_{0} + \frac{\norm{\vg_{0}}}{\gamma \beta_{0}},
    \end{align*}
    and if $t_0 \ge 2$, then
    \begin{align*}
        \beta_{t_0} &\le \beta_{t_0-1} + \frac{\norm{\vg_{t_0-1}}}{\gamma \beta_{t_0-1}} 
        = \beta_{t_0-1} + \frac{2L^2}{\gamma} \frac{\norm{\vg_{t_0-2}}}{\beta_{t_0-1}^3}
        + \frac{2}{\gamma} \frac{\norm{\vg_{t_0-2}}}{\beta_{t_0-2}} \\
        &\le \beta_{t_0-1} + \frac{2L^2}{\gamma} \frac{l(\beta_{t_0-1} - \beta_{t_0-2})\beta_{t_0-2}}{\beta_{t_0-1}^3}
        + \frac{2}{\gamma} l(\beta_{t_0-1} - \beta_{t_0-2}) \\
        &\le L + \frac{2l}{\gamma\beta_0}L^2 + \frac{2l}{\gamma} L.
    \end{align*}
    At last, for $t_0 > 0$, we have
        \begin{align*}
        F(\vx_{t_0-1}) - F(\vx_0) &\le
        \sum_{i=0}^{t_0-2} - \frac{1}{\beta_{i+1}}\left(1 - \frac{L}{2\beta_{i+1}} \right) \norm{\vg_i} \\
        &\le \frac{L}{2} \sum_{i=0}^{t_0-2} \frac{\norm{\vg_i}}{\beta_{i+1}^2}
        \le \frac{L}{2} \sum_{i=0}^{t_0-2} \frac{\norm{\vg_i}}{\beta_{i}^2} \\
        &= \frac{L}{2} \sum_{i=0}^{t_0-2} \varphi' \left(\frac{\beta_{i+1}}{\beta_i}\right) 
        \le \frac{Ll}{2} \sum_{i=0}^{t_0-2} \left(\frac{\beta_{i+1}}{\beta_i} - 1\right) \\
        &\le \frac{Ll}{2} \sum_{i=0}^{t_0-2} \left(\frac{\beta_{i+1} - \beta_i}{\beta_0}\right)
        = \frac{Ll}{2\beta_0} (\beta_{t_0-1} - \beta_0).
    \end{align*}
\end{proof}

\begin{proof}[\textbf{proof of Theorem \ref{scalar-2-thm}}]
    The proof is completely similar to the proof of Theorem \ref{scalar-1-thm}.
\end{proof}

\end{document}